\documentclass[lettersize,journal]{IEEEtran}
\usepackage{amsmath,amssymb,amsfonts}
\usepackage{algorithmic}
\usepackage{algorithm}
\usepackage{amsmath, bm, mathrsfs}
\usepackage{array}
\usepackage[caption=false,font=normalsize,labelfont=sf,textfont=sf]{subfig}
\usepackage{textcomp}
\usepackage{hyperref}
\usepackage{stfloats}
\usepackage{url}
\usepackage{verbatim}
\usepackage{amsthm}   
\newtheorem{proposition}{Proposition}
\newtheorem{theorem}{Theorem}
\newtheorem{lemma}{Lemma}
\newtheorem{definition}{Definition}
\newtheorem{assumption}{Assumption}
\newtheorem{corollary}{Corollary}
\newtheorem{remark}{Remark}
\usepackage{graphicx}
\usepackage{cite}
\usepackage{xcolor}
\usepackage{siunitx}
\usepackage{booktabs}
\usepackage{pbox}
\usepackage{makecell}
\usepackage{mathrsfs}
\begin{document}

\title{Quantitative Error Feedback for Quantization Noise Reduction of Filtering over Graphs}

\author{
\IEEEauthorblockN{Xue Xian Zheng,~\IEEEmembership{Graduate Student Member,~IEEE}, Weihang Liu, Xin Lou,~\IEEEmembership{Senior Member,~IEEE}, \\Stefan Vlaski,~\IEEEmembership{Member,~IEEE}, and Tareq Y. Al-Naffouri,~\IEEEmembership{Fellow,~IEEE}}
\thanks{An earlier version \cite{Zhengconf} of this paper was presented in part at the 50th IEEE International Conference on Acoustics, Speech and Signal Processing (ICASSP), 2025, Hyderabad, India.
}
\thanks{Xue Xian Zheng, Tareq Y. Al-Naffouri are with the Division of Computer, Electrical and Mathematical Sciences and Engineering (CEMSE), King Abdullah University Of Science And Technology (KAUST), Thuwal
23955-6900, Kingdom of Saudi Arabia.}
\thanks{Weihang Liu, Xin Lou are with the School of Information Science and
Technology, ShanghaiTech University, Shanghai 201210, China.}
\thanks{Stefan Vlaski is with the Department of Electrical and Electronic Engineering, Imperial College London, SW7 2AZ London, U.K.}}

%
\markboth{Accepted by IEEE Transactions on Signal Processing
}%
{Shell \MakeLowercase{\textit{et al.}}: A Sample Article Using IEEEtran.cls for IEEE Journals}
\maketitle

\begin{abstract}
This paper introduces an innovative error feedback framework designed to mitigate quantization noise in distributed graph filtering, where communications are constrained to quantized messages. It draws from
error spectrum shaping techniques from state-space digital filters, and
therefore establishes connections between quantized filtering processes
over different domains. In contrast to existing error compensation methods, our framework quantitatively feeds back the quantization noise for exact compensation. We examine the framework under three key scenarios: (i) deterministic graph filtering, (ii) graph filtering over random graphs, and (iii) graph filtering with random node-asynchronous updates. Rigorous theoretical analysis demonstrates that the proposed framework significantly reduces the effect of quantization noise, and we provide closed-form solutions for the optimal error feedback coefficients. Moreover, this quantitative error feedback mechanism can be seamlessly integrated into communication-efficient decentralized optimization frameworks, enabling lower error floors. Numerical experiments validate the theoretical results, consistently showing that our method outperforms conventional quantization strategies in terms of both accuracy and robustness.
\end{abstract}


\section{Introduction}
\IEEEPARstart{T}{he} theory of graph filtering has seen substantial progress in recent years \cite{Zhengconf,Pierre1,Mourasig,Pierre2}, emerging as a cornerstone of modern signal processing and machine learning on networked data. A rich array of graph filters \cite{Elvinoverview}---spanning spectral methods to spatial aggregation schemes---has been developed, providing foundational insight into the analysis, modeling, and transformation of signals defined over irregular domains. These tools have proven invaluable across a wide range of applications, including graph signal denoising \cite{denoising}, clustering \cite{Pierreclustering}, compression \cite{compression}, and reconstruction \cite{reconstruction}, as well as network topology learning and inference \cite{topology}. Furthermore, graph filters form the backbone of graph neural networks \cite{defferrard2016convolutional,kipf2016semi,luo2024classic,arroyo2025vanishing}, empowering advances in neuroscience \cite{neuroscience}, finance \cite{stock}, drug discovery \cite{drug}, and beyond. 

Despite their success, practical implementations often face critical challenges when scaling to large-scale graphs and large data volumes. Centralized deployments \cite{gama2020graphs} could become computationally prohibitive, as they require simultaneous access to and processing of the entire graph structure and all associated data attributes. Distributed implementations \cite{vlaski2023networked} address this by allowing each node to locally process the data and share its information with its immediate neighbors, thus improving the scalability of the computation. For instance, finite impulse response (FIR) graph filters \cite{Mario1} employ iterative aggregation of multi-hop neighbor information for localized computations, while infinite impulse response (IIR) graph filters \cite{Elvin_iir} leverage continuous neighbor communication to perform recursive node updates. Nevertheless, these benefits of distributed graph filtering are accompanied by significant communication overhead due to frequent message exchanges. 

While low-bit quantization \cite{Rouladiff} can reduce transmission costs, the resulting quantization noise can accumulate and propagate across the entire network, complicating the analysis and quantification of performance degradation. This issue is further exacerbated by the inherent randomness of link losses \cite{Elvinrandom} and asynchronous operations \cite{Tekeas,Coutino} in distributed systems, making noise management particularly challenging.

In this paper, an innovative error feedback framework is proposed to address these challenges. The key insight lies in drawing the analogy between the problem of quantized message passing in graph filtering and the classical round-off noise problem in state-space digital filtering. By reimagining error spectrum shaping (ESS) techniques, the framework enables exact compensation of quantization errors, achieving accurate and communication-efficient filtering processes over graphs. Furthermore, it can be extended from filtering to optimization over graphs, improving optimization accuracy under communication constraints. 

\subsection{Related Works}
\subsubsection{Noise Reduction in Quantized Graph Filtering}
Early research efforts primarily examined quantization effects in consensus algorithms---essentially DC-graph filters---where nodes iteratively converge to global averages under communication constraints \cite{Rabbat, Hongbin, Gossip, Soummya, Dorina}. More recent work expanded these insights to general graph filters, addressing quantization challenges through diverse strategies. For example, \cite{Alejandro} analyzes the impact of fixed-stepsize quantization on FIR graph filters, while \cite{Nobre} introduces an adaptive quantization scheme that optimizes bit allocation to minimize errors. Another line of work in \cite{Dorina2} approximates graph spectral dictionaries via polynomials of the graph Laplacian, enhancing robustness to signal quantization through learned polynomial representations. The state-of-the-art approach in \cite{Elvinqq} dynamically reduces quantization step sizes while co-optimizing filter designs. Nevertheless, these techniques do not incorporate post-quantization mitigation strategies explicitly designed to actively compensate for quantization distortions after occurrence, limiting their potential effectiveness.
\subsubsection{Error Feedback in Quantized VLSI Digital Filtering} Error feedback mechanisms have long been pivotal for mitigating quantization noise in digital signal processing, especially within fixed-point very-large-scale integration (VLSI) systems \cite{parhi2007vlsi}. Early innovations, such as ESS \cite{ESS} for pulse-code modulation (PCM) signal quantization, laid foundational principles for noise reduction in recursive digital filters \cite{Bede}. Subsequent developments incorporated error feedback modules into state-space digital filters via simple circuit implementations \cite{Vaidyanathan, Munson, Hinamoto, Lu}, enabling detailed control of quantization artifacts. Parallel advancements in $\Delta$-$\Sigma$ modulation have utilized noise-shaping techniques, leveraging specialized feedback filters to reposition quantization noise outside critical signal bands \cite{ERRORD1, ERRORD2, ERRORD3}. Despite substantial progress, these methodologies remain restricted to classical digital filtering and circuit designs. To date, no research has effectively translated these error feedback concepts into graph filtering frameworks to achieve precise compensation for communication noises.
\subsubsection{Error Feedback for Decentralized Optimization under Communication Constraints} Error feedback strategies have recently become prominent in enhancing communication efficiency in distributed optimization scenarios, particularly addressing gradient quantization artifacts \cite{seide14_interspeech, stich2018sparsified, ICML}. Approaches proposed by \cite{richtarik2021ef21, richtarikMomentum, beznosikov2023biased} effectively employ quantization residuals, reinjecting them into subsequent gradient compression to preserve stable convergence. However, these techniques predominantly focus on centralized parameter-server architectures, overlooking the unique constraints of decentralized, graph-based systems. Recent works \cite{deff,def} have demonstrated impressive progress in integrating error feedback into diffusion-based decentralized optimization with differential quantization. Nevertheless, certain challenges remain—particularly in the tuning of noise compensation hyperparameters, which may not consistently lead to optimal performance. Moreover, error feedback mechanisms for non-gradient compensation remain unexplored.

\subsection{Paper Contributions}
This paper makes the following key contributions:
\begin{enumerate}
    \item We introduce a state-space perspective for graph filtering as a bridge to classical state-space digital filtering. Leveraging this perspective, we reimagine ESS—a technique from VLSI digital filtering—for graph domains, modeling the noise from quantized communications as round-off errors. By deriving noise propagation models and optimizing the noise gain, our method achieves the exact compensation of quantization noise.
    
    \item We conduct a thorough analysis on quantization noise of graph filtering processes in deterministic settings, over random graphs, and with random node-asynchronous updates. Moreover, we provide closed-form solutions for error feedback coefficients in these scenarios, and rigorously show that exact compensation can be achieved regardless of whether filtering is performed over deterministic or random processes.
    
    \item We demonstrate that our proposed approach can be seamlessly integrated into communication-efficient decentralized optimization frameworks. Specifically, we extend our framework to address the challenges of regression over networks under quantized communication. Our results show that our approach achieves significant reductions in quantization noise while improving optimization accuracy. These findings highlight its potential for robust and efficient decentralized optimization in resource-constrained environments.
\end{enumerate}

\subsection{Organization and Notation}
The remainder of this paper is organized as follows. Section II introduces the state-space perspective of graph filtering and reviews fundamental concepts in quantization for communication efficiency. Section III conducts a quantization noise analysis of deterministic graph filtering processes and proposes a quantitative error feedback mechanism to achieve exact compensation, enabling optimal noise reduction. Section IV extends this analysis to graph filtering under random graphs and node-asynchronous updates, demonstrating how exact error compensation via feedback remains effective despite these uncertainties. Section V corroborates our theoretical findings through numerical simulations and demonstrates the connection and framework's potential for decentralized optimization via a linear regression task over networks. Finally, Section VI concludes the paper and summarizes key contributions.

Throughout this paper, we use the following notations. Non-boldface lowercase and uppercase letters (e.g., \(x, X\)) denote scalars, boldface lowercase letters (e.g., \(\mathbf{x}\)) represent column vectors, and boldface uppercase letters (e.g., \(\mathbf{X}\)) signify matrices. The operators \((\cdot)^\top\), \((\cdot)^{*}\), \((\cdot)^\mathrm{H}\), and \((\cdot)^{-1}\) indicate the transpose, conjugate, Hermitian transpose, and inverse, respectively. We denote by \(\mathbb{R}\) the set of real numbers and  \(\mathbb{C}\)  the set of complex numbers, respectively. For a complex number \(z\), we use \(\Re[z]\), \(\Im[z]\) and \(\vert z\vert\) to represent its real part, imaginary part, and modulus, respectively. For a vector \(\mathbf{x}\), \([\mathbf{x}]_{i}\) denotes its \(i\)-th element. For a matrix \(\mathbf{X}\), \([\mathbf{X}]_{ij}\) denotes the element in the \(i\)-th row and \(j\)-th column. The trace of \(\mathbf{X}\) is denoted by \(\mathrm{tr}(\mathbf{X})\), while \(\rho(\mathbf{X})\) represents its spectral radius, defined as the largest eigenvalue in magnitude. The spectral norm of \(\mathbf{X}\) is denoted by \(\|\mathbf{X}\|_2\). $\mathrm{diag}(\cdot)$ returns a diagonal
matrix with its argument on the main diagonal. $\mathrm{blkdiag}(\cdot)$ returns a block diagonal
matrix. $\mathrm{vec}(\cdot)$ denotes the vectorization, which transforms a matrix into a single column vector by vertically stacking its columns. $\mathbf{1}$ and $\mathbf{I}$ are the all-one vector and identity matrix of appropriate size, respectively. We use $\mathbb{E}[\cdot]$ to denote the expectation. We use $\succ$ and $\succeq$ to represent the positive definite (PD) and positive semi-definite (PSD) orderings, respectively. The Hadamard and Kronecker product are represented by $\odot$ and $\otimes$, respectively, while $\prod^{\leftarrow}$ and $\prod^{\rightarrow}$ denote left and right continued matrix multiplications.

\section{State-Space Recursion of Graph Filtering and Quantization}
\label{s2}
Consider the following state-space recursion model
 \begin{equation}
\begin{aligned}
    \mathbf{w}_{t} & = \mathbf{\Phi}_{t-1} \mathbf{w}_{t-1} + \mathbf{\Gamma}_{t} \mathbf{u}_t,
\end{aligned}
\label{state-space}
\end{equation}
where \(\mathbf{u}_t \in \mathbb{C}^N\) represents the input vector and \(\mathbf{w}_t \in \mathbb{C}^{N}\) is the state vector containing cumulative and intermediate information from previous inputs. The model parameters \(\mathbf{\Phi}_{t-1}\in\mathbb{C}^{N\times N}\), \(\mathbf{\Gamma}_{t}\in\mathbb{C}^{N\times N}\) are the state transition matrix and the input mapping matrix, respectively. Depending on how \(\mathbf{\Phi}_{t-1}\) is parameterized, (\ref{state-space}) characterizes a range of fundamental dynamics. For instance, in state-space digital filters, \(\mathbf{\Phi}_{t-1}\) is derived from the desired frequency response, encoding pole information that can be implemented in the circuits; here, \(\mathbf{w}_t\) plays the role of intermediate computational nodes stored in hardware registers. In contrast, structured (neural) state-space sequence models \cite{gu2021efficiently,gu2024mamba,dao2024transformers} parameterize \(\mathbf{\Phi}_{t-1}\) using the HiPPO matrix \cite{gu2020hippo} or its diagonal variants, replacing attention mechanisms for capturing long-range dependencies; in these models, \(\mathbf{w}_t\) typically resides in on-chip memory to facilitate tensor-core computations. Graph filtering is another key parametrization of (\ref{state-space}), where \(\mathbf{\Phi}_{t-1}\) encodes the topology of a graph \(\mathcal{G} = (\mathcal{V}, \mathcal{E})\). Here node set \(\mathcal{V} = \{1, 2, \dots, N\}\) represents the graph's vertices, while edges in \(\mathcal{E} \subseteq \mathcal{V} \times \mathcal{V}\) with size \(|\mathcal{E}| = E\) define connections: an edge \((j, i) \in \mathcal{E}\) indicates a link from node \(j\) to node \(i\). In this context, both \(\mathbf{u}_t\) and \(\mathbf{w}_t\) are node-defined signals: \(\mathcal{V} \to \mathbb{C}^N\). Depending on the characteristics of the graph and the type of recursion adopted, the following filtering processes are considered:

\subsection{Deterministic Graph Filtering}

If \(\mathcal{G}\) is fixed throughout the aforementioned recursions, we have deterministic graph filtering processes. Let the graph shift operator $\mathbf{S} \in \mathbb{R}^{N \times N}$ encode the graph structure, which is commonly chosen as the adjacency matrix $\mathbf{A}$, the Laplacian $\mathbf{L} = \mathrm{diag}(\mathbf{A}\mathbf{1}) - \mathbf{A}$ or any of their normalized or translated variants. And let us assume it has a bounded spectral radius \(\rho(\mathbf{S})=\rho_{\star}\). We have the following filters to capture different propagation dynamics over the graph: 

\textbf{Deterministic FIR.} By \(\mathbf{\Phi}_{t-1}=\mathbf{S}\), \(\mathbf{u}_t=\mathbf{0}\) and \(\mathbf{w}_0=\mathbf{x}\) in (\ref{state-space}), the following recursion outputs $\mathbf{y}$:
 \begin{equation}
\begin{aligned}
    \mathbf{w}_{t} = \mathbf{S}\mathbf{w}_{t-1},  \quad \textnormal{and} \quad \mathbf{y} =   \sum_{t=0}^{T}\phi_{t}\mathbf{w}_{t}.
\end{aligned}
\label{dfir}
\end{equation}

\textbf{Deterministic IIR.} By \(\mathbf{\Phi}_{t-1}= \psi_k\mathbf{S}\) 
 (with \(|\psi_k \rho_{\star}| < 1\)), \(\mathbf{\Gamma}_{t}=\varphi_k\mathbf{I}\), and \(\mathbf{u}_t= \mathbf{x}\) in (\ref{state-space}), the dynamics and output $\mathbf{y}_t$ are given by:
  \begin{equation}
\begin{aligned}
\mathbf{w}^{(k)}_{t} &= \psi_k \mathbf{S}\mathbf{w}^{(k)}_{t-1} + \varphi_k\mathbf{x} \\
\mathbf{y}_{t} &= \sum_{k=1}^{K}\mathbf{w}^{(k)}_{t} \quad \text{for } t \geq 1.
\end{aligned}
\label{diir}
\end{equation}

 As we can see, (\ref{dfir}) truncates the recursion (\ref{state-space}) to a finite \(T\)-step process, assigning a scalar coefficient \(\phi_t\) to each intermediate state \(\mathbf{w}_t\). In contrast, (\ref{diir}) implements \(K\) parallel fixed-point iterations of (\ref{state-space}), which asymptotically converge to a steady state governed by coefficient pairs \((\psi_k, \varphi_k)\). The exact solutions of (\ref{dfir}) and (\ref{diir}) are given by $\mathbf{y} = \sum_{t=0}^{T} \phi_t \mathbf{S}^t \mathbf{x}$, and \(\mathbf{y}=\lim_{t\to\infty}\mathbf{y}_{t}=\sum_{k=1}^K\varphi_k(\mathbf{I}-\psi_k \mathbf{S})^{-1}\mathbf{x}\), respectively. 

The interpretations for the above solutions come from graph Fourier analysis, based on the following definition.
\begin{definition}[Graph Fourier transform]
\label{D1}
Let the graph shift operator \(\mathbf{S}\in \mathbb{R}^{N \times N}\) be diagonalizable with eigenvalue decomposition (EVD) \(
\mathbf{S} = \mathbf{U}\mathbf{\Lambda}\mathbf{U}^{-1}\), where eigenvectors \(\mathbf{U} = [\mathbf{u}_1, \mathbf{u}_2, \ldots, \mathbf{u}_N]\) represent the graph Fourier modes, and eigenvalues \(\mathbf{\Lambda} = \mathrm{diag}(\lambda_{1}, \dots, \lambda_{N})\) correspond to the graph frequencies (assume real-valued). For any signal \(\mathbf{x} \in \mathbb{R}^{N}\), its graph Fourier transform is defined as \(\hat{\mathbf{x}}=\mathbf{U}^{-1}\mathbf{x}\), and the corresponding inverse graph Fourier transform is given by \(\mathbf{x}=\mathbf{U}\hat{\mathbf{x}}\), where $\hat{\mathbf{x}}$ is the vector of graph Fourier coefficients. 
\end{definition}
Under Definition \ref{D1}, plugging \(
\mathbf{S} = \mathbf{U}\mathbf{\Lambda}\mathbf{U}^{-1}\) into the solutions of (\ref{dfir})-(\ref{diir}) shows FIR graph filtering modulates graph Fourier coefficients $\hat{\mathbf{x}}$ via polynomials \(h(\lambda)=\sum_{t=0}^T \phi_t \lambda^{t}\), while IIR employs rational polynomials of the form \(h(\lambda)=\sum_{k=1}^K \varphi_{k}/(1-\lambda\psi_{k})\), both over the graph frequency set \(\lambda \in \{\lambda_{1}, \dots, \lambda_{N}\}\). As \(h(\lambda)\in\mathbb{R}\), it requires \(\phi_t \in \mathbb{R}\) and conjugate pairs \((\psi_k, \varphi_k)\in \mathbb{C}\). This resembles the celebrated frequency response modulations of classical FIR and IIR digital filters.

\subsection{Graph Filtering over Random Graphs}
If the connectivity of \(\mathcal{G}\) changes randomly over the recursions, we have graph filtering processes over random graphs, which is often used to model graph filtering under random link failures. More concretely,

\begin{definition}[Random edge sampling \cite{Elvinrandom}]
\label{D2}
Let $\mathcal{G} = (\mathcal{V}, \mathcal{E})$ be the underlying graph with edge $(j,i) \in \mathcal{E}$ that may fail. Its realization $\mathcal{G}_t = (\mathcal{V}, \mathcal{E}_t)$ is formed with edge sampling $\text{Pr}[(i,j)\in \mathcal{E}_{t}] = p, 0 \leq p < 1 $ for all $(i,j)\in \mathcal{E}$ independently.
\end{definition}

Under Definition \ref{D2}, the graph shift operator here is random, generated using a random Bernoulli matrix $\mathbf{B}_{t}$ as a mask \cite{Zhan}: $\mathbf{S}_{t} = \mathbf{A}_{t} = \mathbf{B}_{t} \odot \mathbf{A}$ or $\mathbf{S}_{t} = \mathbf{L}_{t} = \mathrm{diag}(\mathbf{A}_{t}\mathbf{1}) - \mathbf{A}_{t}$, with $\overline{\mathbf{S}} = \mathbb{E}[\mathbf{S}_{t}] = p \mathbf{A}$ or $\overline{\mathbf{S}} = \mathbb{E}[\mathbf{S}_{t}] = \mathbb{E}[\mathrm{diag}(\mathbf{A}_{t}\mathbf{1})]- p \mathbf{A}$, respectively. Furthermore, we assume the underlying graph \(\mathcal{G}\) is undirected, which implies that $\mathbf{S}$ is symmetric, i.e., $\mathbf{S}=\mathbf{S}^\top$. These constructions guarantee the spectral norm property $\Vert \mathbf{S}_{t} \Vert_{2} \leq \rho_{\star} $. Then

\textbf{FIR over Random Graphs.} By \(\mathbf{S}=\mathbf{S}_{t-1}\) and \(\mathbf{w}_0= \mathbf{x}\) in (\ref{dfir}), we obtain
 \begin{equation}
\begin{aligned}
    \mathbf{w}_{t} = \mathbf{S}_{t-1}\mathbf{w}_{t-1},  \quad \textnormal{and} \quad \mathbf{y} =   \sum_{t=0}^{T}\phi_{t}\mathbf{w}_{t}.
\end{aligned}
\label{rfir}
\end{equation}

 \textbf{IIR over Random Graphs.} By \( \mathbf{S}= \mathbf{S}_{t-1}\) in (\ref{diir}), satisfying $\Vert \psi_k \mathbf{S}_{t-1} \Vert_{2} \leq |\psi_k\rho_{\star}|< 1$), we have 
  \begin{equation}
\begin{aligned}
\mathbf{w}^{(k)}_{t} &= \psi_k \mathbf{S}_{t-1}\mathbf{w}^{(k)}_{t-1} + \varphi_k\mathbf{x} \\
\mathbf{y}_{t} &= \sum_{k=1}^{K}\mathbf{w}^{(k)}_{t} \quad \text{for } t \geq 1.
\end{aligned}
\label{riir}
\end{equation}

In contrast to prior results, the outputs \(\mathbf{y}\) from (\ref{rfir}) and \(\mathbf{y}_{t}\) from (\ref{riir}) are now random vectors, 
with the filtering process specifically encapsulated in their means. By taking expectations in (\ref{rfir}) and on both sides of (\ref{riir}), and examining the asymptotic behavior, it can be derived that \(\overline{\mathbf{y}}=\mathbb{E}[\mathbf{y}]=\sum_{t=0}^{T} \phi_t\overline{\mathbf{S}}^{t}\mathbf{x}\)
and \(\overline{\mathbf{y}}=\lim_{t\to \infty}{\mathbb{E}[\mathbf{y}_{t}]}=\sum_{k=1}^K\varphi_k(\mathbf{I}-\psi_k \overline{\mathbf{S}})^{-1}\mathbf{x}\). Compared with the solutions of (\ref{dfir})-(\ref{diir}), these results essentially replace \(\mathbf{S}\) with \(\overline{\mathbf{S}}\), which can be interpreted as filtering on the expected graph.
In addition to \(\overline{\mathbf{y}}\), the whole process also introduces notable variances \(Var[\mathbf{y}]\) induced by graph randomness, which are upper bounded by a quantity determined by the edge sampling probability \(p\) \cite[Proposition 3 and Theorem 3]{Elvinrandom}.

\subsection{Graph Filtering with Random Node-Asynchronous Updates}
If the nodes of \(\mathcal{G}\) are randomly updated throughout the recursions, the resulting process constitutes asynchronous graph filtering. This is established for dealing with large-scale graphs, where global synchronization becomes an important limitation. Specifically,  
\begin{definition}[Random node selection\cite{Tekeas}]
\label{D3}
Let $\mathcal{V}=\{1, \ldots, N\}$ be the complete node set of the underlying graph $\mathcal{G}$ with each node that may be selected for a subset. Its subset realization $\mathcal{T}_{t}$ can form an index-selection matrix $\mathbf{P}_{\mathcal{T}_{t}} \in \mathbb{R}^{N \times N}$, which is a diagonal matrix that has value $1$ only at the indices specified by the set $\mathcal{T}_{t}$. That is,
  \begin{equation}
\begin{aligned}
\mathbf{P}_{\mathcal{T}_{t}} = \sum_{i \in \mathcal{T}_{t}} \mathbf{e}_i \mathbf{e}_i^\top, \quad \textnormal{and} \quad \mathrm{tr}(\mathbf{P}_{\mathcal{T}_{t}}) = |\mathcal{T}_{t}|,
\end{aligned}
\label{def_rns1}
\end{equation}
where $\mathbf{e}_i \in \mathbb{R}^N$ is the $i$-th standard vector that has $1$ at the $i$th index and $0$ elsewhere, and $|\mathcal{T}_{t}|$ is the size of $\mathcal{T}_{t}$. And the distribution of $\mathbf{P}_{\mathcal{T}_{t}}$ satisfies 
  \begin{equation}
\begin{aligned}
\mathbb{E}[\mathbf{P}_{\mathcal{T}_t}] = \mathbf{P} \quad \forall t,  \quad \textnormal{and} \quad \mathbf{0} \prec \mathbf{P} \preceq \mathbf{I}.
\end{aligned}
\label{def_rns2}
\end{equation}
Here $\mathbf{P}\in\mathbb{R}^{N \times N}$ is the average index (node) selection matrix, which is a deterministic and diagonal matrix.
\end{definition}

Under Definition \ref{D3}, $\mathcal{T}_t$ identifies the nodes to be updated. Incorporating $\mathcal{T}_t$ into the recursions enables randomized asynchronous variations, where only indices from $\mathcal{T}_t$ are updated at each step, while the remaining indices remain unchanged. Specifically, we have the filter in the following:

\textbf{IIR with Random Node-Asynchronous Updates.} The randomized (\ref{diir}) updates $[\mathbf{w}^{(k)}_{t}]_i$ as $[\psi_k \mathbf{S} \mathbf{w}^{(k)}_{t-1}]_i + [\varphi_k \mathbf{x}]_i$ for $i \in \mathcal{T}_t$, while it remains $[\mathbf{w}^{(k)}_{t-1}]_i$ for $i \notin \mathcal{T}_t$, which is
  \begin{equation}
\begin{aligned}
\mathbf{w}^{(k)}_{t} & = \big(\mathbf{I}+\mathbf{P}_{\mathcal{T}_t} (\psi_k\mathbf{S} - \mathbf{I}) \big)\mathbf{w}^{(k)}_{t-1} + \mathbf{P}_{\mathcal{T}_t} (\varphi_k\mathbf{x})\\
\mathbf{y}_{t} &= \sum_{k=1}^{K}\mathbf{w}^{(k)}_{t} \quad \text{for } t \geq 1.
\end{aligned}
\label{aiir}
\end{equation}
as $\mathbf{w}^{(k)}_{t} = \sum_{i \notin \mathcal{T}_{t}} \mathbf{e}_i \mathbf{e}_i^\top \mathbf{w}^{(k)}_{t-1}+\sum_{i \in \mathcal{T}_{t}} \mathbf{e}_i \mathbf{e}_i^\top(\psi_k \mathbf{S}\mathbf{w}^{(k)}_{t-1} + \varphi_k\mathbf{x})$.

One can easily verify that the steady-state solution \(\mathbf{y}=\sum_{k=1}^K\varphi_k(\mathbf{I}-\psi_k \mathbf{S})^{-1}\mathbf{x}\) coming from (\ref{diir}) satisfies (\ref{aiir}) by letting $\mathbf{w}^{(k)}_{t}=\mathbf{w}^{(k)}_{t-1}$. Moreover, if \(\Vert \psi_{k}\mathbf{S}\Vert_{2}<1\) and \(\mathbf{P}=p\mathbf{I}\) with $\text{Pr}[i\in\mathcal{T}_{t}] = p, 0 < p \leq 1 $ for all node $i\in\mathcal{V}$, then asymptotically \(\lim_{t\to \infty}{\mathbb{E}[\mathbf{y}_{t}]}\to \mathbf{y}\) and \(\lim_{t\to \infty}{\mathbb{E}[\Vert\mathbf{y}_{t}-\mathbf{y}\Vert^{2}_{2}]}\to \mathbf{0}\), which ensures the mean square convergence of (\ref{aiir}) \cite[Corollary 2]{Tekeas}. In contrast to (\ref{riir}) for random graphs, which filters in the mean $\overline{\mathbf{y}}$ while incurring notable variance \(Var[\mathbf{y}]\), this filtering converges to the exact solution of synchronized deterministic filtering, despite the randomness induced by node selection.\footnote{Definition \ref{D3} can also be applied to construct a random node-synchronous FIR graph filter. However, this approach does not guarantee convergence; instead, it yields a filter similar to (\ref{rfir}) that operates on the mean but with notable variance. An alternative convergent asynchronous update FIR filter may be found in \cite{Coutino}.}

\subsection{Quantization Model}
The overall filtering methods described above are primarily derived from parameterizing the state transition matrix \(\mathbf{\Phi}_{t-1}\) with the graph shift operator \(\mathbf{S}\) and its variants. Broadly, these matrices belong to a family of \emph{combination matrices} in the context of decentralized optimization, which impose structural sparsity constraints solely based on the underlying graph connectivity. Specifically, if nodes \(i\) and \(j\) are disconnected in the graph, the corresponding entry in the matrix satisfies \([\mathbf{\Phi}_{t-1}]_{ij} = 0.\) This sparsity structure enables a distributed and collaborative implementation, where each node aggregates information only from its local neighborhood. To formalize this, let $\mathcal{N}_i = \{j \mid (j, i) \in \mathcal{E}\}$ denote the set of nodes directly connected to node $i$ (including $i$ itself), it holds that \([\mathbf{\Phi}_{t-1}\mathbf{w}_{t}]_{i}\to\sum_{j\in \mathcal{N}_i}[\mathbf{\Phi}_{t-1}]_{ij}[\mathbf{w}_{t}]_{j}\) for recursion (\ref{state-space}). A critical challenge arises from the communication overhead associated with exchanging $[\mathbf{w}_{t}]_{j}$ between nodes. Because both energy and bandwidth resources are limited, nodes resort to quantizing their transmissions to reduce the bit rate. This introduces communication noise, which can be described by

\begin{assumption}[Communication noise by randomized message quantization\cite{Nobre,Elvinqq}]
\label{Ap1}
Let the combination process be modified for communication efficiency as
\(
\mathbf{\Phi}_{t-1}\mathbf{w}_t \longrightarrow \mathbf{\Phi}_{t-1}\,\mathcal{Q}[\mathbf{w}_t],
\)
where \(\mathcal{Q}[\cdot]\) denotes a uniformly randomized quantizer operating over the interval \([ -r, r ]\) with a \(b\)-bit signed representation. The quantization operation is modeled as
\begin{equation}
\mathcal{Q}[\mathbf{w}_t] = \mathbf{w}_t + \mathbf{n}_{t},
\label{assum_q1}
\end{equation}
where the noise \(\mathbf{n}_{t}\in\mathbb{C}^{N}\) being i.i.d. for each $t$, satisfies,
\begin{equation}
 \mathbb{E}[\mathbf{n}_{t}] = \mathbf{0} \quad \textnormal{and} \quad  \Sigma_{\mathbf{n}_t} = \sigma^2 \mathbf{I} = \frac{\Delta^2}{12} \Big(1 + \mathbb{I}_{\Im[\mathbf{n}_{t}]\neq\mathbf{0}}\Big) \mathbf{I}.
\label{assum_q2}
\end{equation}
Here $\Delta = 2r/2^b$ is the stepsize and $\mathbb{I}$ is the indicator function on $\Im[\mathbf{n}_{t}]$, respectively.
\end{assumption}
Integrating this assumption into recursions (\ref{dfir})-(\ref{riir}) and (\ref{aiir}) yields quantized graph filtering processes. Analyzing such processes is challenging as quantization noise is not only introduced but also accumulates and propagates across the entire network. However, from the illustrated state-space perspective, the communication noise here resembles the round-off noise in the context of digital filtering, where finite-precision arithmetic introduces quantization noise that amplifies and propagates through recursive filter structures. Based on this connection, we have the following remark:

\begin{remark}[Equalizing dynamic ranges with node variant graph filters]
\label{Rm1}
\textnormal{It is often desirable to assign the same bits to each node while avoiding overflow in the quantizer $\mathcal{Q}[\cdot]$ with range $[-r, r]$. The method of coordinate transformations \cite{parhi2007vlsi,Lu} for state-space recursion states that there exists a deterministic diagonal matrix $\mathbf{T}$ such that, through the equivalent state-space formulation $\mathbf{w}'_t = \mathbf{T}^{-1}\mathbf{w}_t$, $\mathbf{\Phi}_{t-1}' = \mathbf{T}^{-1}\mathbf{\Phi}_{t-1}\mathbf{T}$, and $\mathbf{\Gamma}' = \mathbf{T}^{-1}\mathbf{\Gamma}_{t}$ with constant $\mathbf{\Phi}_{t-1}, \mathbf{\Gamma}_{t}$, it holds:
\begin{equation}
    \mathbb{E}\left[ \left| \big[\mathbf{w}'_{t}\big]_{i} \right|^{2} \right] \leq r^2, \quad \forall i, t
    \label{eq:dynamic_range}
\end{equation}
for any energy-constrained input signal. This establishes an $\ell_2$-norm dynamic constraint on transmitted messages to regulate quantization. When normalization is applied to the input signal, $r$ can be set to $1$ for stable state-space recursion. For graph filters \(\mathbf{\Phi}_{t-1}\) parameterized by a graph shift operator $\mathbf{S}$ and its variants, the output relationship becomes: $\mathbf{y} = \sum_{t=0}^{T}\phi_{t}\mathbf{T}\mathbf{w}'_{t}$ and $\mathbf{y}_{t}= \sum_{k=1}^{K}\mathbf{T}^{(k)}\mathbf{w}'^{(k)}_{t}$. This provides an important insight for node-variant graph filters \cite{Mario1}: one can always design \emph{node‐energy‐balanced} graph filters by applying a suitable node‐variant (diagonal) transformation.}
\end{remark}

The developments in this section establish how graph filtering can be viewed through a state‐space lens, with communication noise manifesting as round-off noise that propagates through recursive updates. As we shall see next, these insights naturally pave the way for error feedback mechanisms to mitigate the quantization noise, and they provide a foundation for seamless integration into decentralized optimization frameworks for communication efficiency.

\section{Noise Analysis and Quantitative Error Feedback on Deterministic Processes}
\label{s3}
This section analyzes the propagation of quantization noise in deterministic graph filtering processes and introduces a quantitative error feedback (QEF) approach to mitigate its effect. Central to our contribution is addressing a critical question for graph filtering: \emph{How much noise should be fed back to achieve exact compensation and maximize global noise reduction?} Specifically, we introduce an auxiliary diagonal error feedback matrix in the network recursion and optimize its coefficients for
optimal noise reduction. This can be interpreted as each node locally storing its quantization error, applying an optimal weight factor, and participating in the combination process for the exact compensation. We derive closed-form solutions for the feedback coefficients for both deterministic FIR and IIR graph filtering, and quantify the exact noise reductions resulting from the error feedback. Moreover, the analysis provides precise noise gain regularizers that enable the design of filters with reduced sensitivity to quantization noise. 

\subsection{QEF on Deterministic FIR Graph Filtering}
\label{s3a}
Under Assumption \ref{Ap1}, the quantized (\ref{dfir}) can be rewritten as 
 \begin{equation}
\begin{aligned}
    \hat{\mathbf{w}}_{t} = \mathbf{S}\mathcal{Q}[\hat{\mathbf{w}}_{t-1}],  \quad \textnormal{and} \quad \mathbf{y} =   \sum_{t=0}^{T}\phi_{t}\hat{\mathbf{w}}_{t},
\end{aligned}
\label{qdfir}
\end{equation}
where initial $\hat{\mathbf{w}}_{0}=\mathbf{x}$. By substituting $\mathcal{Q}[\hat{\mathbf{w}}_{t-1}]=\hat{\mathbf{w}}_{t-1} + \mathbf{n}_{t-1}$ into (\ref{qdfir}) and subtracting (\ref{dfir}) from (\ref{qdfir}), we have the noise propagation model
 \begin{equation}
\begin{aligned}
    \widetilde{\mathbf{w}}_{t} &= \mathbf{S}(\widetilde{\mathbf{w}}_{t-1}+\mathbf{n}_{t-1}),  \quad \textnormal{and} \quad \widetilde{\mathbf{y}} = \sum_{t=0}^{T}\phi_{t}\widetilde{\mathbf{w}}_{t}, 
\end{aligned}
\label{eqdfir}
\end{equation}
where $\widetilde{\mathbf{w}}_{t} = \hat{\mathbf{w}}_{t} -\mathbf{w}_{t}$ and $\widetilde{\mathbf{w}}_{0}=\mathbf{0}$. If we introduce an auxiliary matrix $\mathbf{D}_{t-1}$ for $\mathbf{n}_{t-1}$ in (\ref{eqdfir}), we have error fed back noise propagation model
  \begin{equation}
\begin{aligned}
 \widetilde{\mathbf{w}}_{t} &= \mathbf{S}(\widetilde{\mathbf{w}}_{t-1}+\mathbf{n}_{t-1})-\mathbf{D}_{t-1}\mathbf{n}_{t-1}, \\
\widetilde{\mathbf{y}} &= \sum_{t=0}^{T}\phi_{t}\widetilde{\mathbf{w}}_{t}
\end{aligned}
\label{efdfir}
\end{equation}
with $\mathbf{n}_{t-1}$ i.i.d for $t$, $\mathbb{E}[\mathbf{n}_{t-1}]=\mathbf{0}$ and $\Sigma_{\mathbf{n}_{t-1}} = \sigma^2 \mathbf{I} $.

In (\ref{efdfir}), $\mathbf{D}_{t-1}$ acts as an error memory unit, strategically shaping the quantization noise to reduce its variance at the output. While setting \(\mathbf{D}_{t-1} = \mathbf{S}\) would theoretically nullify the output noise, this is infeasible in practice due to the local accessibility of quantization noise. Therefore, \(\mathbf{D}_{t-1}\) is constrained to be a diagonal matrix. This parallels ESS techniques in traditional signal processing, where feedback filters reshape the noise spectrum to minimize distortion. 

We can differently rewrite (\ref{efdfir}) for considering the contributions from each $\mathbf{n}_{t-1}$ as 
  \begin{equation}
\begin{aligned}
 \widetilde{\mathbf{y}} &= \sum^{T}_{t=1}\widetilde{\mathbf{y}}_{\mathbf{n}_{t-1}} =\sum^{T}_{t=1}\underbrace{\sum_{\tau =t}^{T}\phi_{\tau}\mathbf{S}^{\tau - t}}_{\mathbf{H}_{\mathbf{n}_{t-1}} (\mathbf{S})}(\mathbf{S}-\mathbf{D}_{t-1})\mathbf{n}_{t-1},
\end{aligned}
\label{deqn_effFIR}
\end{equation}
where $\mathbf{H}_{\mathbf{n}_{t-1}}(\mathbf{S})$ is the sub-graph filter that may amplify the noise. The visualization of (\ref{deqn_effFIR}) can be found in \cite{Zhengconf}. To achieve optimal noise reduction, we have the following results:

\begin{proposition}
\label{prop1}
For the quantization noise output with feedback $\widetilde{\mathbf{y}}_{\mathbf{n}_{t-1}} = \mathbf{H}_{\mathbf{n}_{t-1}} (\mathbf{S})(\mathbf{S}-\mathbf{D}_{t-1})\mathbf{n}_{t-1}$ coming from $\mathbf{n}_{t-1}$ with sub-graph filter $\mathbf{H}_{\mathbf{n}_{t-1}} (\mathbf{S}) =  \sum_{\tau =t}^{T}\phi_{\tau}\mathbf{S}^{\tau - t}$, its average noise power per node $ \zeta_{t-1} =\frac{1}{N} \mathrm{tr}\Big( \mathbb{E}[\widetilde{\mathbf{y}}_{\mathbf{n}_{t-1}} \widetilde{\mathbf{y}}^\mathrm{H}_{\mathbf{n}_{t-1}}]\Big) $ satisfies
\begin{equation}
\begin{aligned}
    \zeta_{t-1} & = \frac{\sigma^2}{N} \bigg\{\underbrace{\mathrm{tr}\Big(\mathbf{G}_{\mathbf{n}_{t-1}}(\mathbf{S}) \mathbf{S}\mathbf{S}^\top)}_{\text{Original noise gain:} G_{t-1}}+ \\
    & \underbrace{\mathrm{tr}\Big(\mathbf{G}_{\mathbf{n}_{t-1}}(\mathbf{S}) \mathbf{D}_{t-1}^{2}\Big)-2\mathrm{tr}\Big(\mathbf{G}_{\mathbf{n}_{t-1}}(\mathbf{S}) \mathbf{S}\mathbf{D}_{t-1}\Big)}_{\text{Mitigation:} I_{t-1}<0} \bigg\}
\end{aligned}
\label{p1}
\end{equation}
where $\mathbf{G}_{\mathbf{n}_{t-1}}(\mathbf{S})=\mathbf{H}^\top_{\mathbf{n}_{t-1}} (\mathbf{S})\mathbf{H}_{\mathbf{n}_{t-1}} (\mathbf{S})$ is the Gram matrix.
\end{proposition}
\begin{proof}[Sketch of proof]
      By simple substitution and $\mathbb{E}[\mathbf{n}_{t-1}\mathbf{n}^\mathrm{H}_{t-1}]=\sigma^2\mathbf{I}$, it holds $\mathbb{E}[\widetilde{\mathbf{y}}_{\mathbf{n}_{t-1}} \widetilde{\mathbf{y}}^\mathrm{H}_{\mathbf{n}_{t-1}}]=\sigma^2(\mathbf{H}_{\mathbf{n}_{t-1}}(\mathbf{S})(\mathbf{S}-\mathbf{D}_{t-1})(\mathbf{S}-\mathbf{D}_{t-1})^\mathrm{H}\mathbf{H}^\mathrm{H}_{\mathbf{n}_{t-1}} (\mathbf{S}))$. Then for a real-valued matrix, \(\mathbf{X^\mathrm{H}}=\mathbf{X^\top}\). By using trace cyclic property \(\mathrm{tr}(\mathbf{XY}) = \mathrm{tr}(\mathbf{YX})\) and symmetric property \(\mathrm{tr}(\mathbf{X}) = \mathrm{tr}(\mathbf{X^\top})\) for $\mathbf{D}_{t-1}$, (\ref{p1}) follows.
\end{proof}
\begin{theorem}[Exact noise reduction for deterministic FIR graph filtering]
\label{Th1}
Under Proposition \ref{prop1}, let $\mathbf{D}_{t-1}=\mathrm{diag}(\alpha_{1, t-1}, \dots, \alpha_{N, t-1})$, $\Theta = \{\mathbf{D}_{0}\mathbf{1}, \dots, \mathbf{D}_{t-1}\mathbf{1}\} \in \mathbb{R}^{N \times T}$, the closed-form solution $\alpha_{i, t-1}$ and noise reduction $I(\Theta)$ are
\begin{equation}
\begin{aligned}
\alpha_{i,t-1} & = \frac{\Big [\mathbf{G}_{\mathbf{n}_{t-1}}(\mathbf{S})\mathbf{S}\Big ]_{ii}}{\Big [ \mathbf{G}_{\mathbf{n}_{t-1}}(\mathbf{S})\Big ]_{ii}},\\
I(\Theta) & =\sum_{t=1}^{T}\frac{\sigma^2}{N}\sum_{i=1}^{N}{\frac{\Big [\mathbf{G}_{\mathbf{n}_{t-1}}(\mathbf{S})\mathbf{S}\Big ]^{2}_{ii}}{\Big [ \mathbf{G}_{\mathbf{n}_{t-1}}(\mathbf{S})\Big ]_{ii}}}.
\end{aligned}
\label{t1}
\end{equation}
\end{theorem}
\begin{proof}
See Appendix~\ref{app:t1}.
\end{proof}

Theorem \ref{Th1} states that arbitrarily feeding back the error does not inherently guarantee noise reduction. For effective noise suppression, the matrix $\mathbf{D}_{t-1}$ must instead be carefully selected to satisfy the condition $I_{t-1}<0$ by (\ref{p1}). Also, there exists an optimal $\mathbf{D}_{t-1}$ that minimizes the overall quantization noise energy. Compared with the error compensation methods that correct the bias introduced in the most recent update, this approach achieves optimal noise reduction by minimizing the cumulative noise gain over the network’s recursions. It is important to clarify the computational overhead of deriving $\mathbf{D}_{t-1}$. The calculation relies on the filter's Gram matrix $\mathbf{G}_{\mathbf{n}_{t-1}}(\mathbf{S})$, a term that encapsulates the spectral response of the sub-graph filter. The primary computational burden is associated with the synthesis of the filter itself (specifically, the evaluation of matrix polynomials or spectral decompositions required to shape the frequency response), a step inherently required by any general-form graph filter design \cite{Mario1}. Crucially, our proposed method does not incur this computational cost anew; instead, it exploits the filter response already synthesized during the design phase. With the pre-computed $\mathbf{G}_{\mathbf{n}_{t-1}}(\mathbf{S})$ available, the subsequent derivation of the optimal $\mathbf{D}_{t-1}$ requires only diagonal element extraction and summation, scaling as $\mathcal{O}(|\mathcal{E}|)$. Consequently, the deployment of $\mathbf{D}_{t-1}$ entails only element-wise operations, ensuring a strict $\mathcal{O}(1)$ complexity per node.

To facilitate scalability, the solution can be further simplified for spatially or temporally invariant cases. For instance, in the node-invariant case where we constrain $\alpha_{i,t-1} \to \alpha_{t-1}$, we obtain $\mathbf{D}_{t-1} = \alpha_{t-1}\mathbf{I}$, reducing the parameter space $\Theta$ from $\mathbb{R}^{N \times T}$ to $\mathbb{R}^{T}$. This is ideal for large-scale graphs where node-specific storage is prohibitive. Alternatively, constraining $\alpha_{i,t-1} \to \alpha_{i}$ yields $\mathbf{D}_{t-1} = \mathrm{diag}(\alpha_{1}, \dots, \alpha_{N})$ with $\Theta \in \mathbb{R}^{N}$. Here, coefficients capture node-specific characteristics but remain constant over time. Imposing sparsity on these coefficients can further reduce storage and computational load.

\begin{remark}[A precise noise gain regularizer for FIR graph filter design]
\label{Rm2}
    \textnormal{While (\ref{p1})-(\ref{t1}) provide an optimal noise reduction, the final output noise power also depends on the noise gain $G_{t-1}$ attributed to the sub-graph filter $\mathbf{H}_{\mathbf{n}_{t-1}} (\mathbf{S})$. Since we have established that the filter synthesis is the primary offline computational step, we can further leverage this phase to constrain $G_{t-1}$ explicitly.
 Let $\boldsymbol{\phi}=[\phi_{0},\dots,\phi_{T}]^\top$ and let $\mathbf{H}_{\mathbf{n}_{t-1}} (\mathbf{S})\mathbf{S}$ be written in vectorized form as
\begin{equation}
\begin{aligned}
\mathrm{vec}(\mathbf{H}_{\mathbf{n}_{t-1}} (\mathbf{S})\mathbf{S})=\underbrace{\big[\mathbf{0},\dots,\mathbf{0}^{t},\mathrm{vec}(\mathbf{S}),\dots,\mathrm{vec}(\mathbf{S}^{T-t+1})\big]}_{\mathbf{V}_{t-1}\in \mathbb{R}^{N^2\times(T+1)}}\boldsymbol{\phi}.
\end{aligned}
\label{reg_fir}
\end{equation}
By using the trace property $\mathrm{tr}(\mathbf{XY})=\mathrm{vec}(\mathbf{X}^\top)^\top\mathrm{vec}(\mathbf{Y})$, $G_{t-1}$ from (\ref{p1}) is rewritten as $G_{t-1} =\boldsymbol{\phi}^\top(\mathbf{V}_{t-1}^\top\mathbf{V}_{t-1})\boldsymbol{\phi}$. We then define the noise gain regularizer as the aggregate gain over all stages
\begin{equation}
\begin{aligned}
\mathcal{R}(\boldsymbol{\phi})=\sum^{T}_{t=1}G_{t-1}=\boldsymbol{\phi}^\top\left(\sum^{T}_{t=1}\mathbf{V}^\top_{t-1}\mathbf{V}_{t-1}\right)\boldsymbol{\phi}.
\end{aligned}
\label{rreg_fir}
\end{equation}
As we can see, $\mathcal{R}(\boldsymbol{\phi})$ in (\ref{rreg_fir}) is a Tikhonov regularizer. Plugging it as $\gamma\mathcal{R}(\boldsymbol{\phi})$ into the least square (LS) design yields closed-form solutions, where the regularization parameter $\gamma >0$ explicitly governs the trade-off between design accuracy and the noise gain toward quantization}
\end{remark}

\subsection{QEF on Deterministic IIR Graph Filtering}
\label{s3b}
Under Assumption \ref{Ap1} and using similar transformations from (\ref{qdfir})-(\ref{efdfir}), the error fed back noise propagation model of (\ref{diir}) can be expressed as 
\begin{equation}
\begin{aligned}
\widetilde{\mathbf{w}}^{(k)}_{t} &= \psi_k \mathbf{S}\widetilde{\mathbf{w}}^{(k)}_{t-1}+(\psi_k \mathbf{S}-\mathbf{D}^{(k)})\mathbf{n}^{(k)}_{t-1}\\
\widetilde{\mathbf{y}}_{t} & = \sum_{k=1}^{K}\widetilde{\mathbf{w}}^{(k)}_{t} \quad \text{for } t \geq 1,
\end{aligned}
\label{efdiir}
\end{equation}
with $\widetilde{\mathbf{w}}^{(k)}_{0}=\mathbf{0}$, $\mathbf{n}^{(k)}_{t-1}$ being i.i.d. uniform distribution for $k$ and $t$, satisfying $\mathbb{E}[\mathbf{n}^{(k)}_{t-1}]=\mathbf{0}$ and $\Sigma_{\mathbf{n}^{(k)}_{t-1}} = (\sigma^{(k)})^{2} \mathbf{I} $. Note that here error feedback matrix $\mathbf{D}^{(k)}$ is time-independent. For the optimal noise reduction, we have 
\begin{proposition}
\label{prop2}
    For the quantization noise model with feedback $\widetilde{\mathbf{w}}^{(k)}_{t} = \psi_k \mathbf{S}\widetilde{\mathbf{w}}^{(k)}_{t-1}+(\psi_k \mathbf{S}-\mathbf{D}^{(k)})\mathbf{n}^{(k)}_{t-1}$, its asymptotic average noise power per node can be represented as $ \zeta^{(k)}_{t \to \infty} =\frac{1}{N} \mathrm{tr}\Big( \lim_{t\to\infty}\mathbb{E} [\widetilde{\mathbf{w}}^{(k)}_{t} (\widetilde{\mathbf{w}}^{(k)}_{t})^\mathrm{H} ]\Big)$, satisfying
  \begin{equation}
\begin{aligned}
 \zeta^{(k)}_{t \to \infty} & =\frac{(\sigma^{(k)})^2}{N}\mathrm{tr} \Big( (\psi_{k}\mathbf{S}-\mathbf{D}^{(k)})^\mathrm{H}\mathbf{W}^{(k)}_0(\psi_{k}\mathbf{S}-\mathbf{D}^{(k)}) \Big) \\ 
 & =\frac{(\sigma^{(k)})^2}{N}
\Big\{\underbrace{\mathrm{tr}(\vert\psi_{k}\vert^{2}\mathbf{W}^{(k)}_0 \mathbf{S}\mathbf{S}^\top) }_{\textit{Original noise gain:}G^{(k)}} +\\&\underbrace{\mathrm{tr}(\mathbf{W}^{(k)}_0\vert\mathbf{D}^{(k)}\vert^{2})-2\Re\Big[\psi_{k}\mathrm{tr}(\mathbf{W}^{(k)}_0\mathbf{S}\mathbf{D}^{(k)*})\Big]}_{\textit{Mitigation:} I^{(k)}<0} \Big\} 
\end{aligned}
\label{p2}
\end{equation}
where $\mathbf{W}^{(k)}_0$ is the observability Gramian of $k$ branch dynamics, satisfying Lyapunov equation $\mathbf{W}^{(k)}_0 = \vert\psi_{k}\vert^{2}\mathbf{S}^\top\mathbf{W}^{(k)}_0 \mathbf{S}+\mathbf{I}$.
\end{proposition}
\begin{proof}
       See Appendix~\ref{app:p2}.    
\end{proof}
\begin{theorem}[Exact noise reduction for deterministic IIR graph filtering]
\label{Th2}
    Under Proposition \ref{prop2}, let $\mathbf{D}^{(k)}=\mathrm{diag}(\alpha_{1}^{(k)}+\beta_{1}^{(k)}\mathrm{i},\dots, \alpha_{N}^{(k)}+\beta_{N}^{(k)}\mathrm{i})$, $\Theta = \{\mathbf{D}^{(1)}\mathbf{1},\dots, \mathbf{D}^{(K)}\mathbf{1}\}\in\mathbb{C}^{N\times K}$, the closed-form solution $\alpha_{i}^{(k)}$, $\beta_{i}^{(k)}$ and noise reduction $I(\Theta)$:
\begin{equation}
\begin{aligned}
\alpha_{i}^{(k)} = \frac{\Re[\psi_{k}][\mathbf{W}_0^{(k)}\mathbf{S}]_{ii}}{[\mathbf{W}_0^{(k)}]_{ii}}&,\beta_{i}^{(k)} = \frac{\Im[\psi_{k}][\mathbf{W}_0^{(k)}\mathbf{S}]_{ii}}{[\mathbf{W}_0^{(k)}]_{ii}},\\
I(\Theta) =\sum_{k=1}^{K}\frac{(\sigma^{(k)})^2}{N}&\sum_{i=1}^{N}\frac{|\psi_{k}|^{2}[\mathbf{W}_0^{(k)}\mathbf{S}]_{ii}^{2}}{[\mathbf{W}_0^{(k)}]_{ii}}.
\end{aligned}
\label{t2}
\end{equation}
\end{theorem}
\begin{proof}
    The proof is similar to that of Theorem \ref{Th1}.
\end{proof}
Theorem \ref{Th2} demonstrates that exact compensation for quantized deterministic IIR graph filtering is achievable even if the recursion is over an infinite horizon. In contrast to Theorem \ref{Th1}, the feedback matrix $\mathbf{D}^{(k)}$ now includes complex coefficients $\alpha_{1}^{(k)}+\beta_{1}^{(k)}\mathrm{i}$, a consequence of the possible complex-valued term $\psi_k$ arising from partial fraction decomposition. As we may reveal in Section \ref{s4}, Theorem \ref{Th2} is particularly relevant to decentralized optimization, where we may find the optimal feedback coefficients providing lower error floors.

\begin{remark}[A precise noise gain regularizer for IIR graph filter design]
\label{Rm3}
\textnormal{We can also regularize the noise gain $G^{(k)}$ to reduce the overall quantization noise. Let $\boldsymbol{\psi}= [\psi_{1},\dots,\psi_{K}]^\top$,  we obtain the regularizer as the aggregate gain over all parallels 
\begin{equation}
\begin{aligned}
& \mathcal{R}(\boldsymbol{\psi})= \bigg[\sum_{k=1}^{K}\frac{1}{1 - |\psi_{k}|^{2}\lambda^{2}_{1}},\cdots, \sum_{k=1}^{K}\frac{1}{1 - |\psi_{k}|^{2}\lambda_{1}\lambda_{N}},\\
& \cdots,\sum_{k=1}^{K}\frac{1}{1 - |\psi_{k}|^{2}\lambda^{2}_{N}}\bigg]\Big(\left( \mathbf{U} \otimes \mathbf{U} \right)^{-1}\vec{\mathbf{1}}\odot\left( \mathbf{U} \otimes \mathbf{U} \right)^\top\vec{\mathbf{1}}\Big),
\end{aligned}
\label{r3_final}
\end{equation}
where $\vec{\mathbf{1}}=[\mathbf{e}_1^\top,\dots,\mathbf{e}_N^\top]^\top$. A detailed proof for (\ref{r3_final}) can be found in the Appendix~\ref{app:r3}.
One can observe that (\ref{r3_final})  represents a weighted reciprocal regularizer, combining coupled frequency terms \(\lambda_{i}\lambda_{j}\) and uncoupled frequency terms \(\lambda^2_{i}\). To understand its behavior, we investigate 
\begin{equation}
\begin{aligned}
&\Big[\left( \mathbf{U} \otimes \mathbf{U} \right)^{-1}\vec{\mathbf{1}}\odot\left( \mathbf{U} \otimes \mathbf{U} \right)^\top\vec{\mathbf{1}}\Big]_{(i-1)N+i} \\
&= [\mathbf{U}^{-1}\mathbf{U}^{-\top}]_{ii}\odot[\mathbf{U}^\top\mathbf{U}]_{ii}>0,
\end{aligned}
\label{r3_final_e}
\end{equation}
by using $(\mathbf{Z}^\top \otimes\mathbf{X})\mathrm{vec}(\mathbf{Y})=\mathrm{vec}(\mathbf{XYZ})$ and inner product \([\mathbf{X}\mathbf{X^\top}]_{ii}>0\). (\ref{r3_final_e}) shows the weights for uncoupled \(\sum_{k=1}^{K}1/(1 - |\psi_{k}|^{2}\lambda^2_{i})\) are always positive, and when any $|\psi_{k}|^{2}\lambda^2_{i}$ approaches $1$, the whole reciprocal becomes large. Under the stability condition $|\psi_{k}\rho_{\star}| < 1$ where $\psi_{k}$ controls the pole's radius, (\ref{r3_final}) therefore enforces $|\psi_{k}\rho_{\star}|^{2}  \ll 1$ to prevent large noise gain towards quantization. This behavior mirrors the round-off noise minimization of classical IIR digital filters \cite{parhi2007vlsi,Xuexian}, where poles positioned close to the unit circle lead to excessive amplification of such noise.}
\end{remark}

\section{Noise Analysis and Quantitative Error Feedback on Random Processes}
In this section, we extend the analysis in Section \ref{s3} to address graph filtering over random graphs and with random node-asynchronous updates. We rigorously demonstrate that exact error compensation via deterministic feedback remains robust even in the presence of these stochastic and non-deterministic challenges. First, we derive expressions for noise propagation that explicitly incorporate a feedback mechanism over these random processes. Subsequently, by leveraging convergence theorems from stochastic analysis, we minimize the noise terms to obtain closed-form optimal solutions and introduce their efficient computations. Finally, the analysis is extended to find estimated noise regularizers for filter designs. 

\subsection{QEF on FIR Graph Filtering over Random Graphs}
\label{s4a}
Similarly from (\ref{qdfir})-(\ref{efdfir}), the fed back noise output from (\ref{rfir}) is given by
  \begin{equation}
\begin{aligned}
 \widetilde{\mathbf{y}} &= \sum^{T}_{t=1}\widetilde{\mathbf{y}}_{\mathbf{n}_{t-1}} =\sum^{T}_{t=1}\underbrace{\sum_{\tau =t}^{T}\phi_{\tau}\mathcal{\Phi}_{t:\tau-1}}_{\mathbf{H}_{\mathbf{n}_{t-1}} (\mathcal{\Phi}_{t:T-1})}(\mathbf{S}_{t-1}-\mathbf{D}_{t-1})\mathbf{n}_{t-1},
\end{aligned}
\label{efrfir}
\end{equation}
with $\mathbf{n}_{t-1}$ i.i.d for $t$, $\mathbb{E}[\mathbf{n}_{t-1}]=\mathbf{0}$ and $\Sigma_{\mathbf{n}_{t-1}} = \sigma^2 \mathbf{I} $. Here $
    \mathbf{H}_{\mathbf{n}_{t-1}} (\mathcal{\Phi}_{t:T-1})$ is the stochastic sub-graph filter over random graph processes $\mathcal{\Phi}_{t:T-1}$ with $\mathcal{\Phi}_{t:t'} = \prod_{\tau = t}^{t'\leftarrow}\mathbf{S}_{\tau}$ if $t'\geq t$ and $\mathbf{I}$ if $t'<t$. Our findings are summarized as follows:
\begin{proposition}
\label{prop3}
    For the quantization noise output with feedback $\widetilde{\mathbf{y}}_{\mathbf{n}_{t-1}} = \mathbf{H}_{\mathbf{n}_{t-1}} (\mathcal{\Phi}_{t:T-1})(\mathbf{S}_{t-1}-\mathbf{D}_{t-1})\mathbf{n}_{t-1}$ coming from $\mathbf{n}_{t-1}$ with stochastic sub-graph filter  $
    \mathbf{H}_{\mathbf{n}_{t-1}} (\mathcal{\Phi}_{t:T-1}) =   \sum_{\tau =t}^{T}\phi_{\tau}\mathcal{\Phi}_{t:\tau-1}$, its average noise power per node $ \zeta_{t-1} =\frac{1}{N} \mathrm{tr}\Big( \mathbb{E}[\widetilde{\mathbf{y}}_{\mathbf{n}_{t-1}} \widetilde{\mathbf{y}}^\mathrm{H}_{\mathbf{n}_{t-1}}]\Big)$ satisfies
    \begin{equation}
\begin{aligned}
      & \zeta_{t-1} = \frac{\sigma^2}{N} \bigg\{\underbrace{\mathrm{tr}\Big(\mathbb{E}[\mathbf{G}_{\mathbf{n}_{t-1}}(\mathcal{\Phi}_{t:T-1}) \mathbf{S}_{t-1}\mathbf{S}^\top_{t-1}]\Big)}_{\text{Original noise gain:} G_{t-1}} + \\ & \underbrace{ \mathrm{tr}\Big(\mathbb{E}[\mathbf{G}_{\mathbf{n}_{t-1}}(\mathcal{\Phi}_{t:T-1})] \mathbf{D}_{t-1}^{2}\Big)-2\mathrm{tr}\Big(\mathbb{E}[\mathbf{G}_{\mathbf{n}_{t-1}}(\mathcal{\Phi}_{t:T-1})] \overline{\mathbf{S}}\mathbf{D}_{t-1}\Big) }_{\textit{Mitigation:} I_{t-1}<0}\bigg\}
\end{aligned}
\label{p3}
\end{equation}
where $\mathbf{G}_{\mathbf{n}_{t-1}}(\mathcal{\Phi}_{t:T-1})=\mathbf{H}^\top_{\mathbf{n}_{t-1}} (\mathcal{\Phi}_{t:T-1})\mathbf{H}_{\mathbf{n}_{t-1}} (\mathcal{\Phi}_{t:T-1})$.
\end{proposition}
\begin{proof}
    See Appendix~\ref{app:p3}
\end{proof}
\begin{theorem}[Exact noise reduction for FIR graph filtering over random graphs]
\label{Th3}
    Under Proposition \ref{prop3}, let $\mathbf{D}_{t-1}=\mathrm{diag}(\alpha_{1, t-1},\dots, \alpha_{N, t-1})$, $\Theta =\{\mathbf{D}_{0}\mathbf{1},\dots,\mathbf{D}_{t-1}\mathbf{1}\}\in\mathbb{R}^{N\times T}$, the closed-form solution $\alpha_{i, t-1}$ and noise reduction $I(\Theta)$ are
    \begin{equation}
\begin{aligned}
\alpha_{i,t-1} & = \frac{\Big [\mathbb{E}[\mathbf{G}_{\mathbf{n}_{t-1}}(\mathcal{\Phi}_{t:T-1})] \overline{\mathbf{S}} \Big ]_{ii}}{\Big [\mathbb{E}[\mathbf{G}_{\mathbf{n}_{t-1}}(\mathcal{\Phi}_{t:T-1})] \Big ]_{ii}},\\
I(\Theta) & =\sum_{t=1}^{T}\frac{\sigma^2}{N}\sum_{i=1}^{N}{\frac{\Big [\mathbb{E}[\mathbf{G}_{\mathbf{n}_{t-1}}(\mathcal{\Phi}_{t:T-1})] \overline{\mathbf{S}}\Big ]^{2}_{ii}}{\Big [ \mathbb{E}[\mathbf{G}_{\mathbf{n}_{t-1}}(\mathcal{\Phi}_{t:T-1})]\Big ]_{ii}}}.
\end{aligned}
\label{t3}
\end{equation}
\end{theorem}
\begin{proof}
   The proof is similar to that of Theorem \ref{Th1}.
\end{proof}
Theorem \ref{Th3} states that if $\mathbb{E}[\mathbf{G}_{\mathbf{n}_{t-1}}(\mathcal{\Phi}_{t:T-1})]$ is tractable, then exact compensation is achievable even in the presence of graph randomness. This tractability arises from exploiting the structural properties inherent to the random matrix sequence $\mathcal{\Phi}_{t:T-1}$. By using the linearity of expectation $\mathbb{E}[\mathbf{X}+\mathbf{Y}] = \mathbb{E}[\mathbf{X}]+\mathbb{E}[\mathbf{Y}]$, $\mathbb{E}[\mathbf{G}_{\mathbf{n}_{t-1}}(\mathcal{\Phi}_{t:T-1})]$ can be expressed as $\sum_{\tau_1 =t}^{T}\sum_{\tau_2 =t}^{T}\phi_{\tau_1}\phi_{\tau_2}\mathbb{E}[\mathcal{\Phi}_{t:\tau_1-1}^\top\mathcal{\Phi}_{t:\tau_2-1}]$. Without loss of generality, let $\tau_1 < \tau_2$ and define $\mathcal{\Phi}^\top_{t:t'} = \prod_{\tau = t}^{\rightarrow t'}\mathbf{S}^\top_{\tau}$, we have 
\begin{equation}
\begin{aligned}
&\mathbb{E}\Big[\mathcal{\Phi}_{t:\tau_1-1}^\top\mathcal{\Phi}_{t:\tau_2-1}\Big]  \\&= \mathbb{E}\bigg[\underbrace{\mathbf{S}^\top_{t} \times\cdots \times\mathbf{S}^\top_{\tau_1-1}}_{\mathcal{\Phi}_{t:\tau_1-1}^\top}
\times\underbrace{\mathbf{S}_{\tau_2-1}\times\cdots}_{\mathcal{\Phi}_{\tau_1:\tau_2-1}} \times \underbrace{\mathbf{S}_{\tau_1-1} \times\cdots \times\mathbf{S}_{t}}_{\mathcal{\Phi}_{t:\tau_1-1}}
\bigg]
\end{aligned}
\label{ker}
\end{equation}
which is simplified to $\mathbb{E}\big[\mathbf{S}^\top_{t} \times\cdots \times \mathbb{E}[\mathbf{S}^\top_{\tau_1-1}\overline{\mathbf{S}}^{\tau_2-\tau_1}\mathbf{S}_{\tau_1-1}]\times \cdots\times\mathbf{S}_{t}\big]$ using total expectation \(\mathbb{E}[\mathbf{X}] = \mathbb{E}[\mathbb{E}[\mathbf{X}|\mathbf{Y}]]\). In contrast to the deterministic case in (\ref{p1}), which relies on simple matrix powers, the term $\mathbb{E}[\mathbf{G}_{\mathbf{n}_{t-1}}(\mathcal{\Phi}_{t:T-1})]$ involves a concatenated forward-backward Markovian chain of nested expectations, as shown in (\ref{ker}). Crucially, observing that this chain exhibits a recursive structure allows us to formulate the following efficient computation method:

\begin{remark}[Exact expectation computation]
 \label{Rm4}
    \textnormal{For any $\tau_1,\tau_2$, we can initialize $\mathbf{M}_{0}=\overline{\mathbf{S}}^{|\tau_2-\tau_1|}$, and recursively calculate the equation $\mathbf{M}_{l}=\mathbb{E}[\mathbf{S}^\top_{t'}\mathbf{M}_{l-1}\mathbf{S}_{t'}]$ until $l= \textnormal{min} \{\tau_1,\tau_2\}-t$, to have $\mathbf{M}_{l}\to\mathbb{E}[\mathcal{\Phi}_{t:\tau_1-1}^\top\mathcal{\Phi}_{t:\tau_2-1}]$ for (\ref{ker}). This $\mathbf{S}_{t'}$ can be any realization from Definition \ref{D2}, only being required as independent of $\mathbf{M}_{l-1}$. For simplicity, we let $\mathbf{S}_{t'}=\mathbf{S}_{t}$ due to its independence throughout the recursion. By the vectorization property of $\mathrm{vec}(\mathbb{E}[\mathbf{X}])=\mathbb{E}[\mathrm{vec}(\mathbf{X})]$ and $\mathrm{vec}(\mathbf{XYZ})=(\mathbf{Z}^\top \otimes\mathbf{X})\mathrm{vec}(\mathbf{Y})$, we can then rewrite the recursive relation as
\begin{equation}
\begin{aligned}
\mathrm{vec}(\mathbf{M}_{l}) & = \mathbb{E}\big[\mathrm{vec}(\mathbf{S}^\top_{t}\mathbf{M}_{l-1}\mathbf{S}_{t})\big]\\
& = \mathbb{E}\big[(\mathbf{S}_{t}^\top \otimes \mathbf{S}_{t}^\top)\big]\mathrm{vec}(\mathbf{M}_{l-1}).
\end{aligned}
\label{vec}
\end{equation}
 The significance of (\ref{vec}) lies in  that, once the edge sampling probability $p$ is specified, the matrix $\mathbb{E}[(\mathbf{S}_{t}^\top \otimes \mathbf{S}_{t}^\top)]$ simplifies to a constant, time-invariant matrix. This transforms the stochastic recursion into a computationally efficient process of repeated multiplication by a fixed matrix. To illustrate, consider the adjacency graph shift operator $\mathbf{S}_{t} = \mathbf{A}_{t} = \mathbf{B}_{t} \odot \mathbf{A}$, where $\mathbf{S}_{t}^\top=\mathbf{S}_{t}$ due to symmetry. Here, the expectation simplifies as $\mathbb{E}[(\mathbf{S}_{t}^\top \otimes \mathbf{S}_{t}^\top)]=\mathbb{E}[(\mathbf{B}_{t} \otimes \mathbf{B}_{t})]\odot(\mathbf{A}\otimes\mathbf{A})$ using mixed product property. Because $\mathbf{B}_{t}$ is a symmetric Bernoulli random matrix (entries are i.i.d. with probability $p$), the entries of $\mathbb{E}[(\mathbf{B}_{t} \otimes \mathbf{B}_{t})]$ take only two possible values: $p,p^2$, depending on whether they involve the same or distinct Bernoulli variables. Since $(\mathbf{A}\otimes\mathbf{A})$ is deterministic, $\mathbb{E}[(\mathbf{S}_{t}^\top \otimes \mathbf{S}_{t}^\top)]$ is efficiently computed by scaling $(\mathbf{A}\otimes\mathbf{A})$ with the above $p/p^2$ structure.}
\end{remark}

\begin{remark}[An estimated noise gain regularizer for FIR graph filter design over random graphs]
\label{Rm5}
    \textnormal{In line with the analysis in Section~\ref{s3a}, we propose an estimated noise gain regularizer that bounds the noise gain and simplifies the formulation for FIR graph filter design over random graphs. For the noise gain $G_{t-1}$ defined in (\ref{p3}), we have
\begin{equation}
\begin{aligned}
\mathcal{R}(\boldsymbol{\phi})=\boldsymbol{\phi'}^\top \left(\sum^{T}_{t=1}\mathbf{V'}_{t-1}\right)\boldsymbol{\phi'} \ge \frac{1}{N\rho^{2}_{\star}}\sum^{T}_{t=1}G_{t-1},
\end{aligned}
\label{r5_final}
\end{equation}
where $\boldsymbol{\phi'}=\big[\vert\phi_{0}\vert,\dots,\vert\phi_{T}\vert\big]^\top$, $\mathbf{V'}_{t-1}\in\mathbb{R}^{(T+1)\times (T+1)}$ with
\begin{equation}
\begin{aligned}
\big[\mathbf{V'}_{t-1}\big]_{(i+1)(j+1)}=\begin{cases} 
0,&\quad \forall i,j < t  \\ 
p^{|i-j|}\rho^{i+j-2t}_{\star},&\quad \forall i,j \ge t 
\end{cases}.
\end{aligned}
\label{r555555}
\end{equation} 
A detailed proof for (\ref{r5_final}) can be found in the Appendix~\ref{app:r5}. As seen from (\ref{r5_final}), the noise gain regularizer is a Tikhonov regularizer on the magnitude of the coefficients $\phi_t$ with parameters $p$ and $\rho_{\star}$ governing its behavior. 
}
\end{remark}

\subsection{QEF on IIR Graph Filtering over Random Graphs}
\label{s4b}
We obtain the fed back noise propagation model of (\ref{riir})
\begin{equation}
\begin{aligned}
\widetilde{\mathbf{w}}^{(k)}_{t} &= \psi_k \mathbf{S}_{t-1}\widetilde{\mathbf{w}}^{(k)}_{t-1}+(\psi_k \mathbf{S}_{t-1}-\mathbf{D}^{(k)})\mathbf{n}^{(k)}_{t-1}\\
\widetilde{\mathbf{y}}_{t} & = \sum_{k=1}^{K}\widetilde{\mathbf{w}}^{(k)}_{t} \quad \text{for } t \geq 1,
\end{aligned}
\label{efriir}
\end{equation}
with $\widetilde{\mathbf{w}}^{(k)}_{0}=\mathbf{0}$, $\mathbf{n}^{(k)}_{t-1}$ being i.i.d. satisfying $\mathbb{E}[\mathbf{n}^{(k)}_{t-1}]=\mathbf{0}$ and $\Sigma_{\mathbf{n}^{(k)}_{t-1}} = (\sigma^{(k)})^{2} \mathbf{I} $, identical to (\ref{efdiir}). In this case, we have the following results:
\begin{proposition}
\label{prop4}
    For the quantization noise model with feedback $\widetilde{\mathbf{w}}^{(k)}_{t} = \psi_k \mathbf{S}_{t-1}\widetilde{\mathbf{w}}^{(k)}_{t-1}+(\psi_k \mathbf{S}_{t-1}-\mathbf{D}^{(k)})\mathbf{n}^{(k)}_{t-1}$, its asymptotic average noise power per node can be represented as $ \zeta^{(k)}_{t \to \infty} =\frac{1}{N} \mathrm{tr}\Big( \lim_{t\to\infty}\mathbb{E} [\widetilde{\mathbf{w}}^{(k)}_{t} (\widetilde{\mathbf{w}}^{(k)}_{t})^\mathrm{H}]\Big)$, satisfying
    \begin{equation}
\begin{aligned}
\zeta^{(k)}_{t \to \infty}& =\frac{(\sigma^{(k)})^2}{N}\mathrm{tr} \Big( \mathbb{E}[ (\psi_{k}\mathbf{S}_{t}-\mathbf{D}^{(k)})^\mathrm{H}\mathbf{W}^{(k)}_{\mathcal{\Phi}}(\psi_{k}\mathbf{S}_{t}-\mathbf{D}^{(k)})]\Big)\\
& = \frac{(\sigma^{(k)})^2}{N}\Big\{
\underbrace{\mathrm{tr}(|\psi_{k}|^{2}\mathbb{E} [\mathbf{W}^{(k)}_{\mathcal{\Phi}} \mathbf{S}_{t}\mathbf{S}^\top_{t}]) }_{\textit{Original noise gain:}G^{(k)}}+ \\
&\underbrace{\mathrm{tr}(\mathbf{W}^{(k)}_{\mathcal{\Phi}}|\mathbf{D}^{(k)}|^{2})-2\Re\Big[\psi_{k}\mathrm{tr}(\mathbf{W}^{(k)}_{\mathcal{\Phi}}\overline{\mathbf{S}}\mathbf{D}^{(k)*})\Big]}_{\textit{Mitigation:} I^{(k)}<0} \Big\} 
\end{aligned}
\label{p4}
\end{equation}
where $\mathbf{W}^{(k)}_{\mathcal{\Phi}}=\sum_{\tau=0}^{\infty}\mathbb{E} [|\psi_{k}|^{2\tau}\mathcal{\Phi}_{0:\tau-1}^\top\mathcal{\Phi}_{0:\tau-1}]$.
\end{proposition}
\begin{proof}
    See Appendix~\ref{app:p4}.
\end{proof}
\begin{lemma}[Convergence of noise gain over random graphs]
\label{lm1}
    If the graph shift operator \(\mathbf{S}_{t}\) coming from any graph realization $\mathcal{G}_{t}(\mathcal{V},\mathcal{E}_{t})$ satisfies $\Vert \psi_k \mathbf{S}_{t} \Vert_{2} \leq |\psi_k\rho_{\star}|< 1$, then infinite matrix series $\mathbf{W}^{(k)}_{\mathcal{\Phi}}=\sum_{\tau=0}^{\infty}\mathbb{E} [|\psi_{k}|^{2\tau}\mathcal{\Phi}_{0:\tau-1}^\top\mathcal{\Phi}_{0:\tau-1}]$ from stochastic system converges, and satisfies 
\begin{equation}
\begin{aligned}
\mathbf{W}^{(k)}_{\mathcal{\Phi}} &= |\psi_{k}|^{2}\mathbb{E}[\mathbf{S}^\top_{t}\mathbf{W}^{(k)}_{\mathcal{\Phi}} \mathbf{S}_{t}]+\mathbf{I}.
\end{aligned}
\label{l1}
\end{equation}
\end{lemma}
\begin{proof}
    See Appendix~\ref{app:l1}.
\end{proof}
\begin{theorem}[Exact noise reduction for IIR graph filtering over random graphs]
\label{Th4}
    Under Proposition \ref{prop4} and Lemma \ref{lm1}, let $\mathbf{D}^{(k)}=\mathrm{diag}(\alpha_{1}^{(k)}+\beta_{1}^{(k)}\mathrm{i},\dots, \alpha_{N}^{(k)}+\beta_{N}^{(k)}\mathrm{i})$, $\Theta = \{\mathbf{D}^{(1)}\mathbf{1},\dots, \mathbf{D}^{(K)}\mathbf{1}\}\in\mathbb{C}^{N\times K}$, the closed-form solution $\alpha_{i}^{(k)}$, $\beta_{i}^{(k)}$ and noise reduction $I(\Theta)$ are
    \begin{equation}
\begin{aligned}
\alpha_{i}^{(k)} = \frac{\Re[\psi_{k}][\mathbf{W}_{\mathcal{\Phi}}^{(k)}\overline{\mathbf{S}}]_{ii}}{[\mathbf{W}_{\mathcal{\Phi}}^{(k)}]_{ii}}&,\beta_{i}^{(k)} = \frac{\Im[\psi_{k}][\mathbf{W}_{\mathcal{\Phi}}^{(k)}\overline{\mathbf{S}}]_{ii}}{[\mathbf{W}_{\mathcal{\Phi}}^{(k)}]_{ii}}, \\
I(\Theta) =\sum_{k=1}^{K}\frac{(\sigma^{(k)})^2}{N}&\sum_{i=1}^{N}\frac{|\psi_{k}|^{2}[\mathbf{W}_{\mathcal{\Phi}}^{(k)}\overline{\mathbf{S}}]_{ii}^{2}}{[\mathbf{W}_{\mathcal{\Phi}}^{(k)}]_{ii}}.
\end{aligned}
\label{t4}
\end{equation}
\end{theorem}
\begin{proof}
    The proof is similar to that of Theorem \ref{Th1}.
\end{proof}

Theorem \ref{Th4} establishes that exact compensation of quantization noise in IIR graph filtering over random graphs is achievable provided that the second-order moment of the noise process on random graphs exhibits absolute convergence. While the expression in (\ref{l1}) appears intricate, the matrix $\mathbf{W}^{(k)}_{\mathcal{\Phi}}$ can be efficiently computed by fixed-point iteration, with the core computation of iterative multiplication described in (\ref{vec}). 

\begin{remark}[An estimated noise gain regularizer for IIR graph filter design over random graphs]
\label{Rm6} 
\textnormal{We can also find an estimated noise gain regularizer for the above filter, which penalizes the upper bound of noise gain. For the noise gain $G^{(k)}$ defined in (\ref{p4}), the regularizer is derived as
\begin{equation}
\begin{aligned}
 \mathcal{R}(\boldsymbol{\psi})= \sum_{k=1}^{K}\frac{|\psi_{k}|^{2}}{1 - |\psi_{k}|^{2}\rho_{\star}^{2}} \ge \frac{1}{N\rho^{2}_{\star}}\sum^{K}_{k=1}G^{(k)}.
\end{aligned}
\label{r6_final}
\end{equation}
The proof of~\eqref{r6_final} is similar to that of Remark \ref{Rm5}, where $\mathbf{W}^{(k)}_{\mathcal{\Phi}}$ is treated analogously to $\mathbb{E}[\mathbf{G}_{\mathbf{n}_{t-1}}(\mathcal{\Phi}_{t:T-1})]$, but extends the argument by incorporating Lemma \ref{lm1}$'$s inequality
\begin{equation}
\begin{aligned}
\Big\Vert\mathbf{W}^{(k)}_{\mathcal{\Phi}}\Big\Vert_{2}\leq\sum^{\infty}_{\tau=0}|\psi_k \rho_{\star}|^{2\tau}=\frac{1}{1 - |\psi_{k}|^{2}\rho_{\star}^{2}}.
\end{aligned}
\label{r6_inq}
\end{equation}
 Compared with (\ref{r5_final}), $\mathcal{R}(\boldsymbol{\psi})$ in (\ref{r6_final}) diminishes the probability $p$ due to the absence of $\overline{\mathbf{S}}$ in the expression (\ref{l1}). Furthermore, as $|\psi_{k}\rho_{\star}|^{2}  < 1$, $\mathcal{R}(\boldsymbol{\psi})$ is increasing monotonically with respect to $|\psi_{k}|^{2}$. By direct algebraic manipulation, minimizing $\mathcal{R}(\boldsymbol{\psi})$ is equivalent to minimizing $\sum_{k=1}^{K}1/(1 - |\psi_{k}|^{2}\rho_{\star}^{2})$. Incorporating this reciprocal into the filter design will enforce $|\psi_{k}\rho_{\star}|^{2}  \ll 1$, parallel to the mechanism in (\ref{r3_final}).
 }  
\end{remark}

\subsection{QEF on IIR Graph Filtering with Random Node-Asynchro-nous Updates}
\label{s4c}
Unlike in previous cases, the random node-asynchronous updates will also trigger a random node message quantization. Under Assumption \ref{Ap1}, the quantized (\ref{aiir}) is expressed as 
\begin{equation}
    \begin{aligned}
    \hat{\mathbf{w}}^{(k)}_{t} & = \sum_{i \notin \mathcal{T}_{t}} \mathbf{e}_i \mathbf{e}_i^\top \hat{\mathbf{w}}^{(k)}_{t-1}+\sum_{i \in \mathcal{T}_{t}} \mathbf{e}_i \mathbf{e}_i^\top(\psi_k \mathbf{S}\mathcal{Q}[\hat{\mathbf{w}}^{(k)}_{t-1}] + \varphi_k\mathbf{x}).
    \end{aligned}
    \label{qasiir}
\end{equation}
This means that only the nodes selected for updates will have their corresponding aggregated information quantized. Based on (\ref{qasiir}) and similar transformations from (\ref{qdfir})-(\ref{efdfir}), we have the following fed back noise propagation model
\begin{equation}
    \begin{aligned}
    \widetilde{\mathbf{w}}^{(k)}_{t} &= \big(\mathbf{I}+\mathbf{P}_{\mathcal{T}_t} (\psi_k\mathbf{S} - \mathbf{I})\big)\widetilde{\mathbf{w}}^{(k)}_{t-1}+\mathbf{P}_{\mathcal{T}_t} (\psi_k\mathbf{S}-\mathbf{D}^{(k)})\mathbf{n}^{(k)}_{t-1}\\
    \widetilde{\mathbf{y}}_{t} & = \sum_{k=1}^{K}\widetilde{\mathbf{w}}^{(k)}_{t} \quad \text{for } t \geq 1,
    \end{aligned}
    \label{efaiir}
\end{equation}
with $\widetilde{\mathbf{w}}^{(k)}_{0}=\mathbf{0}$, $\mathbf{n}^{(k)}_{t-1}$ i.i.d., $\mathbb{E}[\mathbf{n}^{(k)}_{t-1}]=\mathbf{0}$, $\Sigma_{\mathbf{n}^{(k)}_{t-1}} = (\sigma^{(k)})^{2} \mathbf{I} $.
(\ref{efaiir}) indicates the error feedback is also selective with only updated nodes. Then we have the following results: 
\begin{proposition}
\label{prop5}
    For the quantization noise model with feedback $\widetilde{\mathbf{w}}^{(k)}_{t}= (\mathbf{I}+\mathbf{P}_{\mathcal{T}_t} (\psi_k\mathbf{S} - \mathbf{I}))\widetilde{\mathbf{w}}^{(k)}_{t-1}+\mathbf{P}_{\mathcal{T}_t} (\psi_k\mathbf{S}-\mathbf{D}^{(k)})\mathbf{n}^{(k)}_{t-1}$, its asymptotic average noise power per node can be represented as $ \zeta^{(k)}_{t \to \infty} =\frac{1}{N} \mathrm{tr}\Big( \lim_{t\to\infty}\mathbb{E} [\widetilde{\mathbf{w}}^{(k)}_{t} (\widetilde{\mathbf{w}}^{(k)}_{t})^\mathrm{H} ]\Big)$, satisfying
    \begin{equation}
\begin{aligned}
&\zeta^{(k)}_{t \to \infty} \\
& =\frac{(\sigma^{(k)})^2}{N}\mathrm{tr} \Big((\psi_{k}\mathbf{S}-\mathbf{D}^{(k)})^\mathrm{H}\mathbb{E}[\mathbf{P}_{\mathcal{T}_t}\mathbf{W}^{(k)}_{\mathbf{P}}\mathbf{P}_{\mathcal{T}_t}](\psi_{k}\mathbf{S}-\mathbf{D}^{(k)})\Big)\\
& = \frac{(\sigma^{(k)})^2}{N}\Bigg\{
\underbrace{\mathrm{tr}(|\psi_{k}|^{2}\mathbb{E}[\mathbf{P}_{\mathcal{T}_t}\mathbf{W}^{(k)}_{\mathbf{P}}\mathbf{P}_{\mathcal{T}_t}] \mathbf{S}\mathbf{S}^\top) }_{\textit{Original noise gain:}G^{(k)}}+ \\
&\quad \quad \underbrace{
\begin{aligned}
&\mathrm{tr}\!\left(
\mathbb{E}[\mathbf{P}_{\mathcal{T}_t}
\mathbf{W}^{(k)}_{\mathbf{P}}
\mathbf{P}_{\mathcal{T}_t}]
|\mathbf{D}^{(k)}|^{2}
\right) \\[-0.3em]
&\quad - 2\Re\!\Big[
\psi_{k}\mathrm{tr}\!\left(
\mathbb{E}[\mathbf{P}_{\mathcal{T}_t}
\mathbf{W}^{(k)}_{\mathbf{P}}
\mathbf{P}_{\mathcal{T}_t}]
\mathbf{S}\mathbf{D}^{(k)*}
\right)
\Big]
\end{aligned}
}_{\textit{Mitigation: } I^{(k)}<0} \Bigg\} 
\end{aligned}
\label{p5}
\end{equation}
where  
\begin{equation}
\begin{aligned}
\mathbf{W}^{(k)}_{\mathbf{P}} &= \sum_{\tau=1}^{\infty}\mathbb{E}\bigg[\prod_{t'=1}^{\tau\leftarrow}\big(\mathbf{I}+(\psi^{*}_k\mathbf{S}^\top - \mathbf{I})\mathbf{P}_{\mathcal{T}_{t'}}\big)\times \\
&\prod_{t'=1}^{\rightarrow\tau}\big(\mathbf{I}+\mathbf{P}_{\mathcal{T}_{t'}} (\psi_k\mathbf{S} - \mathbf{I})\big)\bigg]+\mathbf{I}.
\end{aligned}
\label{p5_1}
\end{equation}
\end{proposition}
\begin{proof}
The proof  is similar to that of Proposition \ref{prop4}. The fundamental technique involves exploiting the temporal independence of $\mathbf{P}_{\mathcal{T}_{t}}$, which establishes a Markovian chain structure similar to the random graph scenario.
\end{proof}
\begin{lemma}[Convergence of noise gain for random node-asynchronous updates]
\label{lm2}
    If the index-selection matrix $\mathbf{P}_{\mathcal{T}_t}$ satisfies $\mathbb{E}[\mathbf{P}_{\mathcal{T}_t}] = \mathbf{P}=p\mathbf{I}$ for all $t$, where $0 < p \leq 1 $, and the condition \(\Vert \psi_{k}\mathbf{S}\Vert_{2}<1\) holds, then infinite matrix series $\mathbf{W}^{(k)}_{\mathbf{P}}$ in (\ref{p5_1}) from asynchronous system converges, and satisfies 
\begin{equation}
\begin{aligned}
\mathbf{W}^{(k)}_{\mathbf{P}} & = \mathbb{E}\Big[\big(\mathbf{I}+(\psi^{*}_k\mathbf{S}^\top - \mathbf{I})\mathbf{P}_{\mathcal{T}_t}\big)\mathbf{W}^{(k)}_{\mathbf{P}}\times\\&\big(\mathbf{I}+\mathbf{P}_{\mathcal{T}_t} (\psi_k\mathbf{S} - \mathbf{I})\big)\Big]+\mathbf{I}.
\end{aligned}
\label{l2}
\end{equation}
\end{lemma}
\begin{proof}
    See Appendix~\ref{app:l2}.
\end{proof}

\begin{figure*}[t] 
  \centering
  \begin{minipage}[b]{0.32\textwidth} 
    \centering
    \includegraphics[width=\linewidth]{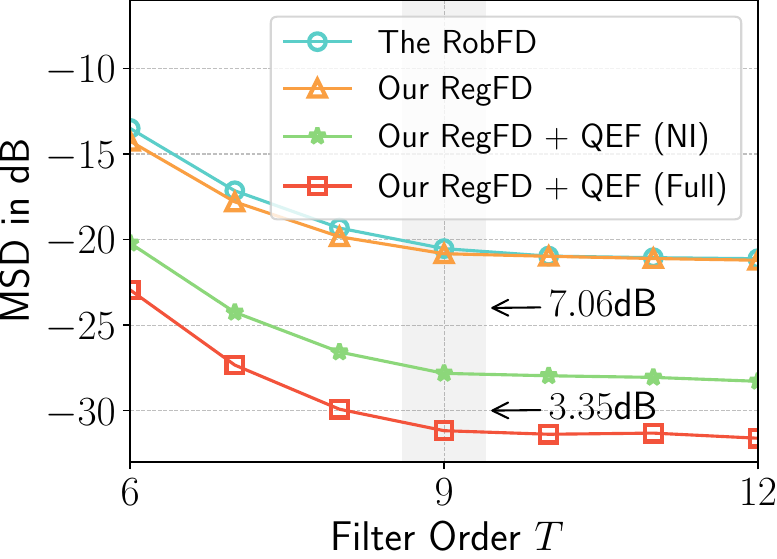}
  \end{minipage}
  \hfill
  \begin{minipage}[b]{0.32\textwidth}
    \centering
    \includegraphics[width=\linewidth]{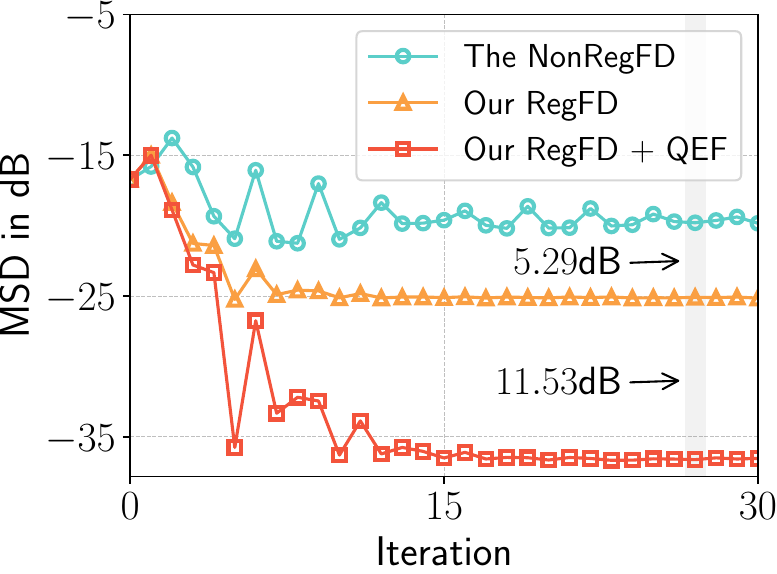}
  \end{minipage}
  \hfill
  \begin{minipage}[b]{0.32\textwidth}
    \centering
    \includegraphics[width=\linewidth]{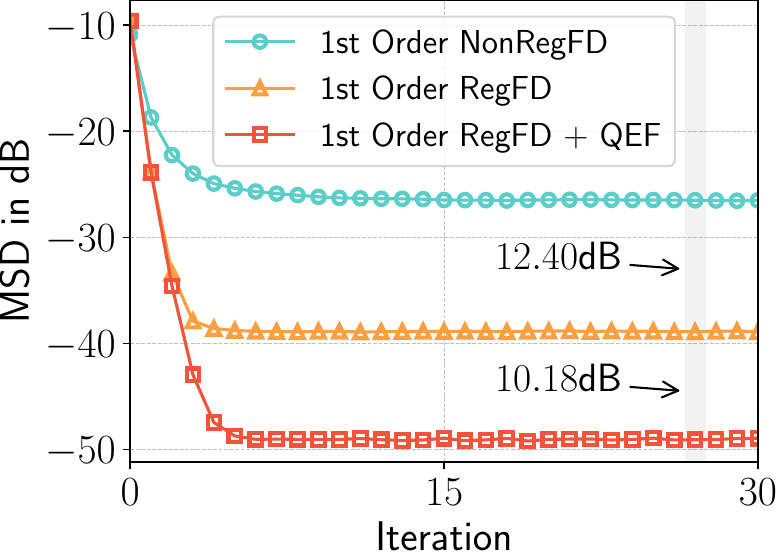}
  \end{minipage}
  \caption{MSD versus filter order/iteration of quantized graph filtering on deterministic processes. {\bf{(Left.)}} FIR with normalized design tolerance $\delta=0.03$ and bit $b=6$; RobFD is the robust design from \cite{Elvinqq}, while QEF is evaluated in two configurations: the Node Invariant (NI) parameterization and the full node-variant optimal solution derived in Theorem \ref{Th1}. {\bf{(Middle.)}} IIR of order $K=3$ with normalized design tolerance $\delta=0.055$, and bit $b=3$; the regularizer weight is set as $\gamma=0.2$ and initial \((\psi_k, \varphi_k)\) comes from \cite{Elvin_iir}. {\bf{(Right.)}} 1st-order IIR with real coefficients: NonRegFD \((\psi_1=-0.974, \varphi_1=-1.414)\) and RegFD \((\psi_1=-0.427, \varphi_1=-1.414)\); this panel corresponds to the real-valued parallel of the Middle panel's IIR graph filter.}
  \label{fig:de}
\end{figure*}

\begin{theorem}[Exact noise reduction for IIR graph filtering with random node-asynchronous updates]
\label{Th5}
    Under Proposition \ref{prop5} and Lemma \ref{lm2}, let $\mathbf{D}^{(k)}=\mathrm{diag}(\alpha_{1}^{(k)}+\beta_{1}^{(k)}\mathrm{i},\dots, \alpha_{N}^{(k)}+\beta_{N}^{(k)}\mathrm{i})$, $\Theta = \{\mathbf{D}^{(1)}\mathbf{1},\dots, \mathbf{D}^{(K)}\mathbf{1}\}\in\mathbb{C}^{N\times K}$, the closed-form solution $\alpha_{i}^{(k)}, \beta_{i}^{(k)}$ and noise reduction $I(\Theta)$ are
    \begin{equation}
\begin{aligned}
\alpha_{i}^{(k)}  & = \frac{\bigg[\Re\Big[\psi_{k}\mathbb{E}\big[\mathbf{P}_{\mathcal{T}_t}\mathbf{W}^{(k)}_{\mathbf{P}}\mathbf{P}_{\mathcal{T}_t}\big]\mathbf{S}\Big]\bigg]_{ii}}{\Big[\mathbb{E}\big[\mathbf{P}_{\mathcal{T}_t}\mathbf{W}^{(k)}_{\mathbf{P}}\mathbf{P}_{\mathcal{T}_t}\big]\Big]_{ii}},\\ \beta_{i}^{(k)}& = \frac{\bigg[\Im\Big[\psi_{k}\mathbb{E}\big[\mathbf{P}_{\mathcal{T}_t}\mathbf{W}^{(k)}_{\mathbf{P}}\mathbf{P}_{\mathcal{T}_t}\big]\mathbf{S}\Big]\bigg]_{ii}}{\Big[\mathbb{E}\big[\mathbf{P}_{\mathcal{T}_t}\mathbf{W}^{(k)}_{\mathbf{P}}\mathbf{P}_{\mathcal{T}_t}\big]\Big]_{ii}},\\
I(\Theta)& =\sum_{k=1}^{K}\frac{(\sigma^{(k)})^2}{N}\sum_{i=1}^{N}\frac{\bigg|\Big[\psi_{k}\mathbb{E}\big[\mathbf{P}_{\mathcal{T}_t}\mathbf{W}^{(k)}_{\mathbf{P}}\mathbf{P}_{\mathcal{T}_t}\big]\mathbf{S}\Big]_{ii}\bigg|^{2}}{\Big[\mathbb{E}\big[\mathbf{P}_{\mathcal{T}_t}\mathbf{W}^{(k)}_{\mathbf{P}}\mathbf{P}_{\mathcal{T}_t}\big]\Big]_{ii}}.
\end{aligned}
\label{t5}
\end{equation}
\end{theorem}
\begin{proof}
    The proof is similar to that of Theorem \ref{Th1}.
\end{proof}
Theorem \ref{Th5} establishes that exact compensation of quantization noise is attainable in IIR graph filtering with random node-asynchronous updates, provided the system adheres to uniform random node selection protocol. Compared to Theorem \ref{Th4}, the closed-form expressions differ because the matrix $\mathbf{W}^{(k)}_{\mathbf{P}}$ is Hermitian rather than the real symmetric matrix as discussed before. Similarly, $\mathbf{W}^{(k)}_{\mathbf{P}}$ can also be computed through the vectorized fixed-point iteration of (\ref{l2}), whose key parameter matrix follows from the expectation
    \begin{equation}
\begin{aligned}
&\mathbb{E}\Big[\big(\mathbf{I}+(\psi_k\mathbf{S}^\top - \mathbf{I})\mathbf{P}_{\mathcal{T}_t}\big)\otimes\big(\mathbf{I}+(\psi^{*}_k\mathbf{S}^\top - \mathbf{I})\mathbf{P}_{\mathcal{T}_t}\big)\Big]  \\
& = \mathbf{I}+p\left((\psi_k\mathbf{S}^\top - \mathbf{I})\otimes\mathbf{I}+\mathbf{I}\otimes(\psi^{*}_k\mathbf{S}^\top - \mathbf{I})\right)\\
& + \left((\psi_k\mathbf{S}^\top - \mathbf{I})\otimes(\psi^{*}_k\mathbf{S}^\top - \mathbf{I})\right)\big(\mathbb{E}[\mathbf{P}_{\mathcal{T}_t}\otimes\mathbf{P}_{\mathcal{T}_t}]\big),
\end{aligned}
\label{vec2}
\end{equation}
using the associativity and mixed-product properties of Kronecker products, the linearity of expectation, and the relation $\mathbb{E}[\mathbf{P}_{\mathcal{T}_t}] = p\mathbf{I}$. The decomposition of (\ref{vec2}) isolates the stochasticity into the term $\mathbb{E}[\mathbf{P}_{\mathcal{T}_t}\otimes\mathbf{P}_{\mathcal{T}_t}]$. Given that $\mathbf{P}_{\mathcal{T}_t}$ is a diagonal matrix populated with independent Bernoulli entries, this expectation simplifies to a constant diagonal matrix with entries taking values of $p$ (autocorrelation) or $p^2$ (cross-correlation), thereby rendering the fixed-point iteration computationally tractable via deterministic matrix operations. By applying the same procedure, one can also derive $\mathbb{E}\big[\mathbf{P}_{\mathcal{T}_t}\mathbf{W}^{(k)}_{\mathbf{P}}\mathbf{P}_{\mathcal{T}_t}\big]$, then obtain (\ref{t5}). It should be noted when $p\to1$, (\ref{p5})-(\ref{t5}) return to (\ref{p2})-(\ref{t2}), indicating synchronous update is a special case of the current consideration. Moreover, the noise gain regularizer, in this case, can share the same expression with (\ref{r6_final}) irrespective of $p$.

\begin{remark}[Consistency of noise gain regularizer (\ref{r6_final}) for IIR graph filter design with random node-asynchronous updates]
\label{Rm7}
\textnormal{Following the symmetry assumption $\mathbf{S}=\mathbf{S}^\top$ from random graph processes, the regularizer $\mathcal{R}(\boldsymbol{\psi})$ is identical to (\ref{r6_final}) for the noise gain defined in (\ref{p5}), based on the inequality
\begin{equation}
\begin{aligned}
\Big\Vert\mathbb{E}[\mathbf{P}_{\mathcal{T}_t}\mathbf{W}^{(k)}_{\mathbf{P}}\mathbf{P}_{\mathcal{T}_t}]\Big\Vert_{2}\leq\frac{1}{1 - |\psi_{k}|^{2}\rho_{\star}^{2}}.
\end{aligned}
\label{r7_final}
\end{equation} 
A detailed proof for (\ref{r7_final}) can be found in Appendix~\ref{app:r7}. }
\end{remark}

\section{Numerical Experiments}
\label{s4}
This section presents numerical simulations validating our proposed framework and theoretical analyses. In the first experiment, we examine a quantized low-pass graph filtering task on deterministic and random processes, consistent with the discussion in previous sections. In the second experiment, we extend the framework to a quantized regression task over networks, illustrating the connection and applicability of our QEF approach to decentralized optimization. All experimental setups are implemented in a Python environment compatible with the GSPbox \cite{GSPBOX} and NetworkX \cite{github} libraries.
\subsection{Low-Pass Graph Filtering}
We consider a low-pass graph filtering task on the David Sensor Network \cite{GSPBOX}, which has \(N = 64\) nodes and \(E = 236\) edges. Following the benchmark in \cite{Elvinqq}, both FIR and IIR graph filters are designed to approximate an ideal low-pass filter with frequency response \(h(\lambda) = 1\) for \(\lambda < \lambda_c\) and \(0\) otherwise. The graph shift operator is chosen as \(\mathbf{S} = \lambda_{\text{max}}^{-1} \mathbf{L}\) and the cut-off frequency is set to \(\lambda_c = 0.5\). A randomized dither quantizer ($[-1, 1]$) is utilized for each transmission, ensuring the quantization noise satisfies Assumption \ref{Ap1}~\cite{Gray1,Gray2}. The input $\mathbf{x}$ has a uniform spectrum and is normalized to prevent overflow.
To evaluate filtering performance under quantization, we use the mean square deviation (MSD):
\begin{equation}
\text{MSD} = \frac{1}{N}\mathbb{E}\Big[\big\Vert \mathbf{y}_{q} - \mathbf{y}\big\Vert^{2}_{2}\Big],~~\textnormal{where} ~~ \mathbf{y} = 
\begin{cases} 
\mathbf{y}, & \text{(unbiased)} \\ 
\overline{\mathbf{y}}, & \text{(biased)} 
\end{cases}.
\end{equation}
Here, \(\mathbf{y}_q\) is the quantization noise injected output, while \(\mathbf{y}\) is the output without quantization, as defined in Section~\ref{s2}. In this context, the presence or absence of bias indicates whether the MSD accounts for graph randomness, which introduces variance \(Var[\mathbf{y}]\) unrelated to quantization.

\subsubsection{Deterministic Processes} For the filter design specifications, we uniformly sample $10^3$ frequency points within $\lambda \in [0,1]$ to characterize the ideal low-pass filter. We employed the square error criterion ($\ell_2$ error) for approximation, with $\delta$ defined as the normalized design tolerance, representing the average $\ell_2$ error per frequency point. 

Our regularized filter design (RegFD) is constructed as follows: \textit{a)} In the FIR case, we minimize the regularizer described in Remark \ref{Rm2}, subject to an approximation error constraint bounded by a given tolerance $\delta$. We evaluate this approach across filter orders $T=6,7,\dots,12$. For benchmarking, we compare results with the robust filter design method (RobFD) from \cite{Elvinqq}. \textit{b)} In the IIR case,  we minimize an unconstrained objective augmented with the weighted regularizer from Remark \ref{Rm3}. The regularization weight $\gamma$ is fine-tuned to ensure the approximation error does not exceed $\delta$. We demonstrate this methodology for $K = 3$ parallel IIR filters. To address the nonconvex optimization landscape, we initialize the coefficients $(\psi_k, \varphi_k)$ using precomputed values from \cite{Elvin_iir} and refine them via sequential quadratic programming (SQP). To ensure a stable design, we further impose the constraints $|\psi_k|^2 \leq 0.95$, and $|\varphi_k|^2 \leq 2$. With $\gamma = 0.2$, RegFD yields 
\(
\{(\psi_k, \varphi_k) \mid (0, 1),\, (-0.427, -1.414),\, (0.248 - 0.795\mathrm{i},0.782 - 0.923\mathrm{i},),\, (0.248 + 0.795\mathrm{i}, 0.782 + 0.923\mathrm{i} )\}.
\)
For comparison, the non-regularized design (NonRegFD) results in coefficients
\(
\{(\psi_k, \varphi_k) \mid (0, 1),\, (-0.974, -1.414),\, (0.342 - 0.913\mathrm{i}, 0.748 - 0.893\mathrm{i}),\, (0.342 + 0.913\mathrm{i}, 0.748 + 0.893\mathrm{i})\}.
\)
Finally, Theorem \ref{Th1} and \ref{Th2} are applied to derive the optimal QEF coefficients.

Fig. \ref{fig:de} demonstrates the filtering performances under quantization, with each evaluated through $10^3$ Monte Carlo trials. As illustrated in Fig.\ref{fig:de} left, our RegFD consistently surpasses the RobFD in all FIR design cases \cite{Elvinqq}. Notably, node-invariant QEF yields a peak gain of $7.06$dB at $T=9$, while the full node-variant solution adds another $3.35$dB. This illustrates the trade-off between the storage efficiency and the maximal noise suppression. Turning to the IIR case in Fig.\ref{fig:de} middle, RegFD exhibits superior steady-state performance compared to NonRegFD, delivering a maximum improvement of $5.29$dB across all $K=3$ parallels. When integrated with QEF, the total improvement increases to $16.82$dB, with a maximum single contribution of $11.53$dB arising from QEF's ability to suppress residual quantization noise. Finally, Fig. \ref{fig:de} right isolates the 1st-order real-valued parallel component from the $K=3$ IIR filter; here, RegFD outperforms NonRegFD by $12.40$dB. This finding directly corroborates the analysis in Remark \ref{Rm3}, which highlights that poles located near the unit circle can cause excessive noise amplification. Even in this isolated scenario, QEF maintains its effectiveness, providing an additional $10.18$dB improvement.
\begin{figure*}[t] 
  \centering
  \begin{minipage}[b]{0.32\textwidth} 
    \centering
    \includegraphics[width=\linewidth]{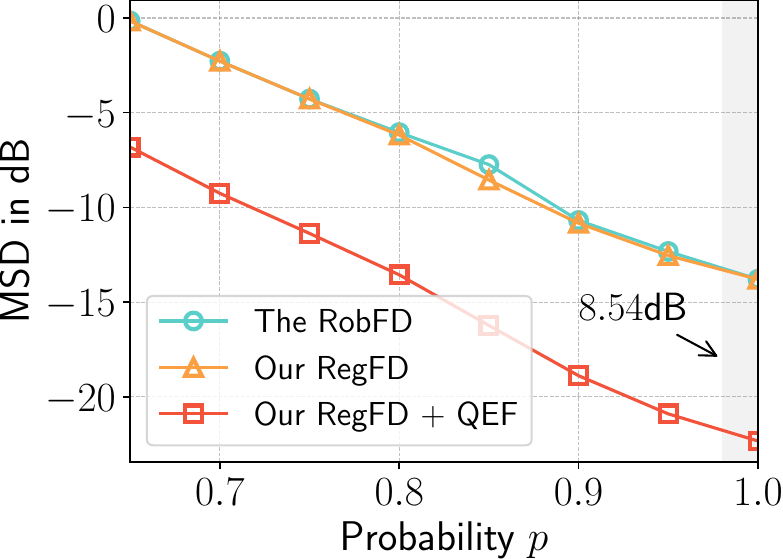}
  \end{minipage}
  \hfill
  \begin{minipage}[b]{0.32\textwidth}
    \centering
    \includegraphics[width=\linewidth]{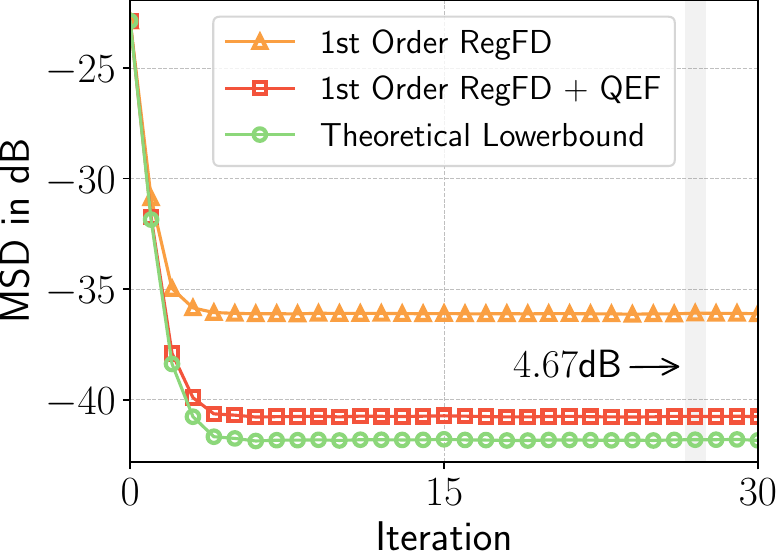}
  \end{minipage}
  \hfill
  \begin{minipage}[b]{0.32\textwidth}
    \centering
    \includegraphics[width=\linewidth]{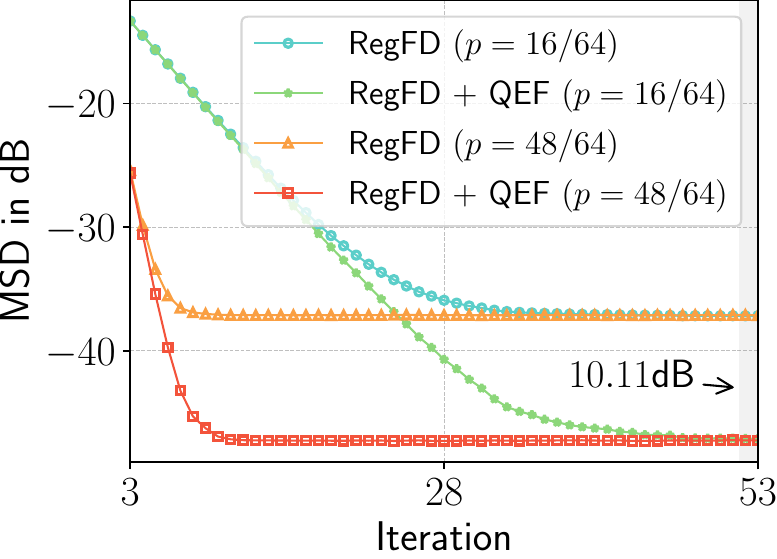}
  \end{minipage}
  \caption{MSD versus probability/iteration of quantized graph filtering on random processes. {\bf{(Left.)}} Unbiased MSD on FIR over random graphs; $\delta=0.03$, $b=6$ and $T=6$; RobFD is the robust design from \cite{Elvinqq}. {\bf{(Middle.)}} Biased MSD on 1st-order IIR \((\psi_1=-0.520, \varphi_1=-1.414)\)  over random graphs; $p=0.95$ and $b=3$. {\bf{(Right.)}} Unbiased MSD on 1st-order IIR \((\psi_1=-0.520, \varphi_1=-1.414)\) with random node-asynchronous
updates; $b=3$.}
  \label{fig:random}
\end{figure*}
\subsubsection{Random Processes} The same filter design specifications are adopted in this scenario. To additionally align with the mean filtering behavior governed by the expected graph shift operator $\overline{\mathbf{S}}=p\mathbf{S}$, we scale the coefficients $\phi_t\to p^t\phi_t$ for FIR filtering over random graphs, leading to a weighted $\ell_2$ design. Meanwhile, we obtain its RegFD by employing the regularizer from Remark \ref{Rm5}, and investigate the case when filter order is set to $T=6$. For IIR filtering over random graphs and with random node-asynchronous updates, we retain $\psi_k$ and use the consistent regularizer from Remark \ref{Rm6} and \ref{Rm7} to obtain their RegFD. With $\gamma=1.5$ and using the same nonconvex optimization strategy in deterministic IIR designs, their RegFD yields 
\(
\{(\psi_k, \varphi_k) \mid (0, 1),\, (-0.520, -1.414),\, (0.221-0.761\mathrm{i}, 0.784-1.044\mathrm{i}),\, (0.221+0.761\mathrm{i}, 0.784+1.044\mathrm{i})\}.
\) For illustration, we shall consider its 1st-order real-valued parallel \((\psi_1=-0.520, \varphi_1=-1.414)\) for the subsequent discussions. Finally, Theorem \ref{Th3}, \ref{Th4} and \ref{Th5} are applied to derive the optimal QEF coefficients for the above filters.

Fig.~\ref{fig:random} demonstrates the quantized filtering performance on random processes, evaluated via $10^6$ Monte Carlo trials accounting for quantization, graph randomness, and asynchrony. The results in Fig.~\ref{fig:random} left show that RegFD with QEF achieves a lower unbiased MSD compared to RobFD, with a maximum improvement of $8.54$dB in the FIR cases. Probability $p$ increasing reduces overall MSD, as power scaling $p^t\phi_t$ diminishes nonregularity on $\phi_t$, thereby mitigating quantization error amplification. When coefficients remain fixed independent of $p$, the quantization error energy exhibits a negative correlation with graph randomness as shown in \cite{Zhengconf}, suggesting that sparser connections inhibit error accumulation during combination. The results in Fig.~\ref{fig:random} middle highlight that even when considering bias from graph randomness, the proposed method maintains a lower deviation than the non-feedback model in the 1st-order IIR, approaching the theoretical lower bound \(Var[\mathbf{y}]\) in the absence of quantization, with a $4.67$dB improvement. The results in Fig.~\ref{fig:random} right underscore QEF's robustness to arbitrary node selection probabilities, where we exemplify two cases $p=16/64, 48/64$, meaning the expected nodes to update are 16 and 48 out of a total of 64 nodes. It shows that the node update probability $p$ primarily affects convergence rates, and QEF results in a maximum steady-state MSD improvement of $10.11$dB for these two cases.


\subsection{Application to Quantized Regression over Networks}
\label{s5b}
Here, we extend the proposed framework to the domain of decentralized optimization, effectively generalizing the graph filter to solving data-dependent stochastic optimization problems. Consider a regression task over a network \cite{Multitaskg,Part1,Part2}
\begin{equation}
\min_{\mathbf{x}} \sum_{i=1}^N f_i(\mathbf{x}_i) + \eta \, \mathbf{x}^\top \mathcal{L} \mathbf{x},
\label{regression}
\end{equation}
where \(\mathbf{x} = [\mathbf{x}_1^\top, \dots, \mathbf{x}_N^\top]^\top \in \mathbb{R}^{MN}\) is the augmented vector concatenating all local variables, each \(f_i(\mathbf{x}_i) = \frac{1}{2}\|\mathbf{A}_i \mathbf{x}_i - \mathbf{b}_i\|_2^2\) is a local LS loss function with \(\mathbf{A}_i \in \mathbb{R}^{L \times M}\), \(\mathbf{x}_i\in\mathbb{R}^M\) and \(\mathbf{b}_i \in \mathbb{R}^L\), and \(\eta\,\mathbf{x}^\top \mathcal{L} \mathbf{x} \,(\eta > 0)\) is a weighted regularizer with \(\mathcal{L}=\mathbf{L}\otimes\mathbf{I}\in \mathbb{R}^{MN\times MN}\) representing the augmented graph Laplacian associated with the network, thus promoting smoothness. One can observe that if the dimension $M\to1$, (\ref{regression}) returns to the graph denoising problem, and its exact solution is given by a 1st-order IIR graph filter. When $M>1$, (\ref{regression}) resorts to a decentralized optimization solution and an \textit{Adapt-then-Combine} (ATC) diffusion strategy can be applied
\begin{equation}
\begin{aligned}
\xi_{i,t} &= \mathbf{x}_{i,t-1} - \mu\nabla f_{i}(\mathbf{x}_{i,t-1})\\
\mathbf{x}_{i,t} &= \sum_{j\in\mathcal{N}_{i}}w_{ij}\xi_{j,t} 
\end{aligned}
\label{ATC}
\end{equation}
where $\mu$ is the step size, $w_{ij}$ is the weight satisfying $w_{ij}=[\mathbf{I}-\mu\eta\mathbf{L}]_{ij}$. (\ref{ATC}) can be understood as the combination of gradient update and graph filtering. The matrix $\mathbf{I}-\mu\eta\mathbf{L}$ specifically acts as a graph filter, where iterated multiplications simulate an IIR. To reduce communication overhead in this filtering, the quantized (\ref{ATC}) endorses $\mathbf{x}_{i,t}=\sum_{j\in\mathcal{N}_{i}}w_{ij}\mathcal{Q}[\xi_{i,t}]$ or $\mathbf{x}_{i,t}=\sum_{j\in\mathcal{N}_{i}}w_{ij}(\hat{\xi}_{i,t-1}+\mathcal{Q}[\xi_{i,t}-\hat{\xi}_{i,t-1}])$ (full state or differential quantization) for efficiency. Under Assumption \ref{Ap1}, where the quantizer is modeled as additive noise, the intermediate estimator $\hat{\xi}_{i,t-1}$ in the differential scheme could cancel out. Consequently, the two methods are analytically equivalent, differing only in the statistical properties of the noise due to the different quantization intervals required for the difference signal versus the full state. Moreover, as the gradient update induced by the LS loss follows a linear recursion, i.e., $\xi_{i,t} = (\mathbf{I} - \mu\mathbf{A}^\top_{i}\mathbf{A}_{i})\mathbf{x}_{i,t-1}+\mu\mathbf{A}^\top_{i}\mathbf{b}_{i}$, we can drive the noise propagation model and noise gain for (\ref{ATC}) similar to previous sections. By denoting $\mathcal{A}=\mathbf{I}-\mu\eta\mathcal{L}$ and $\mathcal{B} =\mathrm{blkdiag}(\mathbf{I} - \mu\mathbf{A}^\top_{1}\mathbf{A}_{1},\dots,\mathbf{I} - \mu\mathbf{A}^\top_{N}\mathbf{A}_{N})$ and using similar transforms from (\ref{qdfir})-(\ref{eqdfir}), we have the following noise propagation model and observability Gramian
\begin{equation}
\begin{aligned}
\boldsymbol{\widetilde{\xi}}_{t} &= \mathcal{B}\mathcal{A}\boldsymbol{\widetilde{\xi}}_{t-1} + \mathcal{B}\mathcal{A}\boldsymbol{n}_{t-1}\\
\mathbf{W}_{\xi} &= \mathcal{A}^\top\mathcal{B}^\top\mathbf{W}_{\xi} \mathcal{B}\mathcal{A} + \mathbf{I}
\end{aligned}
\label{Noiseppp}
\end{equation}
where $\boldsymbol{\widetilde{\xi}}_{t} = [\widetilde{\xi}_{1,t}^\top,\dots,\widetilde{\xi}_{N,t}^\top]^\top$ is the augmented error state vector and $\boldsymbol{n}_{t-1}=[\mathbf{n}_{1,t-1},\dots,\mathbf{n}_{N,t-1}]$ is augmented injected noise with each $i$ node exhibiting $\widetilde{\xi}_{i,t}^\top\,,\mathbf{n}_{i,t-1}$. From (\ref{Noiseppp}), it manifests the identical dynamic as deterministic IIR graph filtering (\ref{efdiir})-(\ref{t2}), therefore we can introduce a block-diagonal error feedback matrix $\mathcal{D}=\mathrm{blkdiag}(\mathbf{D}_{1},\dots,\mathbf{D}_{N})$ with each node feedback $\mathbf{D}_i\in\mathbb{R}^{M\times M}$ to reduce the quantization noise. To align with the conventions of employing scalar weight error feedback for individual nodes (e.g., damping factors or normalization coefficients), we posit $\mathbf{D}_i=\alpha_{i}\mathbf{I}$ as a simplified parameterization to demonstrate methodological advantages. Then, a Corollary derived from Theorem \ref{Th2} (see Appendix \ref{Appb1}) can be directly applied to reconstruct the complete $\mathcal{D}$. Then, we can apply QEF to (\ref{ATC}) as Appendix \ref{Appb2} shows.

For numerical verification, we consider a network of $N=50$ nodes following the benchmark network topology described in~\cite{Multitaskg,Part1,Part2, github}. Each node $i$ observes data pairs $\{\mathbf{A}_i, \mathbf{b}_i\}$ that adhere to a $L=40,\,M=4$ linear regression model $\mathbf{b}_i=\mathbf{A}_i \mathbf{x}^\star_i+\mathbf{v}_{i}$, where $\mathbf{x}^\star_i$ is the true vector to be regressed, $\mathbf{A}_i$ assembles i.i.d. Gaussian entries satisfing $[\mathbf{A}_i]_{lm}\sim\mathcal{N}(0,\sigma^2_{\mathbf{A}_{i}})$, and $\mathbf{v}_{i}\sim\mathcal{N}(\mathbf{0},\sigma^2_{\mathbf{v}_{i}}\mathbf{I})$ is the additive Gaussian noise vector. The regression vector $\mathbf{x}^\star_i$ is pre-smoothed by filtering out high graph frequency components. The matrix and noise parameters are i.i.d. sampled from uniform distributions such as $\sigma_{\mathbf{A}_i}\sim \mathcal{U}[0,1]$ and $\sigma_{\mathbf{v}_i}\sim \mathcal{U}[0,0.1]$ for all nodes $i \in \{1,\ldots,N\}$. The ATC diffusion employs $\eta=10$ and $\mu=0.01$. In terms of quantization, we adapt Assumption \ref{Ap1} to have a uniform probabilistic quantizer with the stepsize $\Delta=10\mu$ and employ a variable coding scheme, consistent with \cite{Rouladiff,def}. The MSD metric is updated by $\sum^N_{i=1}\mathbb{E}[\Vert\mathbf{x}^\star_i-\mathbf{x}_{i,t}\Vert^2_{2}]/N$ for evaluation. Although non-uniform quantization schemes (e.g., log-scale) combined with differential quantization provide known benefits for iterative algorithms \cite{Rouladiff,doostmohammadian2024logarithmically,doostmohammadian2025log,michelusi2022finite}, they are inconsistent with the arithmetic model established in Assumption \ref{Ap1}. Consequently, we reserve the extension of QEF to non-uniform quantization regimes for future work.

\begin{figure}[!tb]
  \centering
  \includegraphics[width=0.8\linewidth]{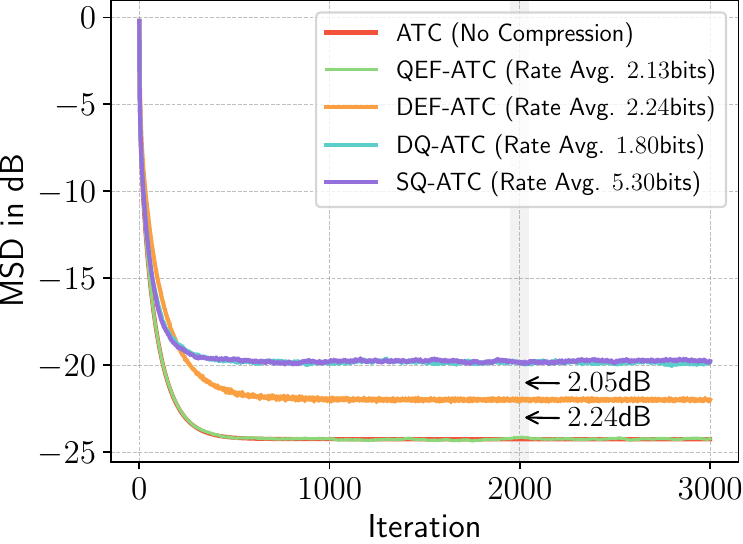}
  \caption{MSD versus iteration of quantized regression over a network, comparing QEF-ATC (the proposed QEF with differential quantization), DEF-ATC\cite{def} (differential error feedback), DQ-ATC \cite{Rouladiff} (differential quantization), and SQ-ATC (full state quantization). The quantizers employ uniform probabilistic quantization with a variable rate coding scheme. The Rate metric (average number (Avg.) of bits per node per component) follows the definition in~\cite{Rouladiff}.}
  \label{fig:layout}
\end{figure}

Fig. \ref{fig:layout} shows the optimization performance under quantization, evaluated through 50 Monte Carlo trials. The results demonstrate that QEF-ATC is the most effective strategy among those tested. By integrating QEF with differential quantization, QEF-ATC almost eliminates the effects of quantization, closely matching the performance of the uncompressed ATC baseline at a modest rate of $ 2.13$ bits. This improvement can be further understood by the significant reduction in noise gain, which stretches from $10^3$ to $10^{-1}$. Moreover, QEF-ATC achieves a $2.24$dB MSD improvement over DEF-ATC \cite{def} (rate: $2.24$ bits, damping factor: $0.6$), suggesting DEF-ATC's hyperparameter tuning for error feedback may not always optimize the performance. While DQ-ATC \cite{Rouladiff} attains the lowest rate ($1.80$ bits), its uncompensated error propagation results in relatively higher MSD. SQ-ATC, on the other hand, suffers from both the highest MSD and the largest rate ($5.30$ bits), highlighting the inefficiency of directly quantizing the full state without differential or error feedback mechanisms.

\section{Conclusion}
In this paper, we proposed an error feedback framework to mitigate quantization noise in distributed graph filtering, leveraging principles from state-space digital filters to systematically quantify and compensate for errors. By deriving closed-form optimal feedback coefficients, the framework achieves precise noise reduction across deterministic, stochastic (random graphs), and node-asynchronous scenarios. Theoretical analysis confirmed significant noise suppression, while integration with decentralized systems demonstrated reduced error floors compared to conventional quantization methods. Numerical experiments validated superior accuracy and robustness in communication-constrained settings. Future work will extend the framework to more general decentralized optimization and large-scale networks. This approach bridges graph signal processing and quantization theory, offering a practical solution for emerging distributed systems where precision and efficiency are critical.

\IEEEtriggeratref{61} 

\vfill

\appendices
\onecolumn
\section{Mathematical Proofs}
In this section, we provide the proofs of our main technical results. These include Theorem \ref{Th1} from Section \ref{s3a}; Proposition \ref{prop2} and Remark \ref{Rm3} from Section \ref{s3b}; Proposition \ref{prop3} and Remark \ref{Rm5} from Section \ref{s4a}; Proposition \ref{prop4} and Lemma \ref{lm1} from Section \ref{s4b}; and Lemma \ref{lm2} and Remark \ref{Rm7} from Section \ref{s4c}.
\subsection{Proof of Theorem 1}
\label{app:t1}
With $\mathbf{D}_{t-1}=\mathrm{diag}(\alpha_{1, t-1},\dots, \alpha_{N, t-1})$, we have the mitigation term
\begin{equation}
\begin{aligned}
    I_{t-1}=-2\sum_{i=1}^{N}\Big[\mathbf{G}_{\mathbf{n}_{t-1}}(\mathbf{S}) \mathbf{S}\Big]_{ii} \alpha_{i, t-1}+ \sum_{i=1}^{N}\Big[\mathbf{G}_{\mathbf{n}_{t-1}}(\mathbf{S})\Big]_{ii}\alpha_{i, t-1}^{2}.  
\end{aligned}
\label{deqn_ex11}
\end{equation}
Let the gradient of $I_{t-1}$ be zero, i.e.,
\begin{equation}
\begin{aligned}
\frac{\partial I_{t-1}}{\partial \alpha_{i, t-1}} =\Big[\mathbf{G}_{\mathbf{n}_{t-1}}(\mathbf{S})\Big]_{ii}\alpha_{i, t-1}-\Big[\mathbf{G}_{\mathbf{n}_{t-1}}(\mathbf{S}) \mathbf{S}\Big]_{ii} = 0,
\end{aligned}
\label{deqn_ex12}
\end{equation}
we have
\begin{equation}
\begin{aligned}
\alpha_{i,t-1} = \frac{\Big [\mathbf{G}_{\mathbf{n}_{t-1}}(\mathbf{S})\mathbf{S}\Big ]_{ii}}{\Big [ \mathbf{G}_{\mathbf{n}_{t-1}}(\mathbf{S})\Big ]_{ii}}, I_{t-1} =\sum_{i=1}^{N}{\frac{\Big [\mathbf{G}_{\mathbf{n}_{t-1}}(\mathbf{S})\mathbf{S}\Big ]^{2}_{ii}}{\Big [ \mathbf{G}_{\mathbf{n}_{t-1}}(\mathbf{S})\Big ]_{ii}}}.
\end{aligned}
\label{deqn_ex13}
\end{equation}
The proof is completed.

\subsection{Proof of Proposition 2}
\label{app:p2}
  By considering all prior noises $\mathbf{n}^{(k)}_{\tau}(\tau<t)$, we can explicitly express 
  \begin{equation}
\begin{aligned}
\widetilde{\mathbf{w}}^{(k)}_{t} = \sum_{\tau=0}^{t-1}(\psi_{k}\mathbf{S})^{t-\tau-1}(\psi_{k}\mathbf{S}-\mathbf{D}^{(k)})\mathbf{n}^{(k)}_{\tau}
  \end{aligned}
\label{p211}
\end{equation}
as $\widetilde{\mathbf{w}}^{(k)}_{0}=\mathbf{0}$. Then, it can be derived 
  \begin{equation}
\begin{aligned}
\mathbb{E} [\widetilde{\mathbf{w}}^{(k)}_{t} (\widetilde{\mathbf{w}}^{(k)}_{t})^\mathrm{H} ] = \mathbb{E}\bigg[\sum_{\tau_1=0}^{t-1}\sum_{\tau_2=0}^{t-1}(\psi_{k}\mathbf{S})^{t-\tau_1-1}(\psi_{k}\mathbf{S}-\mathbf{D}^{(k)})\mathbf{n}^{(k)}_{\tau_1}(\mathbf{n}^{(k)}_{\tau_2})^\mathrm{H} (\psi_{k}\mathbf{S}-\mathbf{D}^{(k)})^\mathrm{H}\big((\psi_k\mathbf{S})^\mathrm{H}\big)^{t-\tau_2-1}\bigg].
  \end{aligned}
\label{p212}
\end{equation}
According to the noise of the i.i.d. $\mathbb{E}[\mathbf{n}^{(k)}_{\tau_1}(\mathbf{n}^{(k)}_{\tau_2})^\mathrm{H}] = (\sigma^{(k)})^2 \mathbf{I}$ if $\tau_1=\tau_2$ and $\mathbf{0}$ if $\tau_1\neq\tau_2$, the trace cyclic property \(\mathrm{tr}(\mathbf{XY}) = \mathrm{tr}(\mathbf{YX})\) and the Hermitian transpose \((\psi_k\mathbf{S})^\mathrm{H}=\psi^{*}_k\mathbf{S}^\top\), it shows
  \begin{equation}
\begin{aligned}
\mathrm{tr}\bigg(\mathbb{E}[\widetilde{\mathbf{w}}^{(k)}_{t}(\widetilde{\mathbf{w}}^{(k)}_{t})^\mathrm{H}]\bigg) = (\sigma^{(k)})^2\mathrm{tr}\bigg(\sum_{\tau=0}^{t-1}(\psi_{k}\mathbf{S}-\mathbf{D}^{(k)})^\mathrm{H}(\psi^{*}_{k}\mathbf{S}^\top)^{t-\tau-1}(\psi_{k}\mathbf{S})^{t-\tau-1}(\psi_{k}\mathbf{S}-\mathbf{D}^{(k)})\bigg).
  \end{aligned}
\label{p213}
\end{equation}
 As $t\to \infty$, one finds $\sum_{\tau=0}^{t-1}(\psi^{*}_{k}\mathbf{S}^\top)^{t-\tau-1}(\psi_{k}\mathbf{S})^{t-\tau-1}\to\sum_{\tau=0}^{\infty}(\psi^{*}_{k}\mathbf{S}^\top)^{\tau}(\psi_{k}\mathbf{S})^{\tau}$. Let $\mathbf{W}^{(k)}_0=\sum_{\tau=0}^{\infty}(\psi^{*}_{k}\mathbf{S}^\top)^{\tau}(\psi_{k}\mathbf{S})^{\tau}$, then $\mathbf{W}^{(k)}_0 $ satisfies Lyapunov equation $\mathbf{W}^{(k)}_0 = \vert\psi_{k}\vert^{2}\mathbf{S}^\top\mathbf{W}^{(k)}_0 \mathbf{S}+\mathbf{I}$, and the limit of (\ref{p213}) can be written as 
   \begin{equation}
\begin{aligned}
\lim_{t\to \infty}\mathrm{tr}\bigg(\mathbb{E}[\widetilde{\mathbf{w}}^{(k)}_{t}(\widetilde{\mathbf{w}}^{(k)}_{t})^\mathrm{H}]\bigg) \to (\sigma^{(k)})^2\mathrm{tr}\bigg((\psi_{k}\mathbf{S}-\mathbf{D}^{(k)})^\mathrm{H}\mathbf{W}^{(k)}_0(\psi_{k}\mathbf{S}-\mathbf{D}^{(k)})\bigg).
  \end{aligned}
\label{p214}
\end{equation}
By expanding R.H.S of (\ref{p214}) with the trace cyclic property \(\mathrm{tr}(\mathbf{XY}) = \mathrm{tr}(\mathbf{YX})\) and the trace-Hermitian property $\mathrm{tr}(\mathbf{X}^\mathrm{H})=\mathrm{tr}(\mathbf{X})^{*}$, we get 
   \begin{equation}
\begin{aligned}
\mathrm{tr}\bigg((\psi_{k}\mathbf{S}-\mathbf{D}^{(k)})^\mathrm{H}\mathbf{W}^{(k)}_0(\psi_{k}\mathbf{S}-\mathbf{D}^{(k)})\bigg)=\mathrm{tr}(\vert\psi_{k}\vert^{2}\mathbf{W}^{(k)}_0 \mathbf{S}\mathbf{S}^\top)+\mathrm{tr}(\mathbf{W}^{(k)}_0\vert\mathbf{D}^{(k)}\vert^{2})-\psi_{k}\mathrm{tr}(\mathbf{W}^{(k)}_0\mathbf{S}\mathbf{D}^{(k)*})-\psi_{k}^{*}\mathrm{tr}(\mathbf{W}^{(k)}_0\mathbf{S}\mathbf{D}^{(k)*})^{*}.
  \end{aligned}
\label{p215}
\end{equation}
Finally, by using conjugate pairing $z_1z_2+z^{*}_1z^{*}_2=2\Re[z_1z_2]$ with $z_1=\psi_{k},z_2=\mathrm{tr}(\mathbf{W}^{(k)}_0\mathbf{S}\mathbf{D}^{(k)*})$ , (\ref{p2}) then follows. The proof is completed.

\subsection{Proof of Remark 3}
\label{app:r3}
With the vectorization property $\mathrm{vec}(\mathbf{X}+\mathbf{Y}) = \mathrm{vec}(\mathbf{X}) +\mathrm{vec}(\mathbf{Y})$ and $\mathrm{vec}(\mathbf{XYZ})=(\mathbf{Z}^\top \otimes\mathbf{X})\mathrm{vec}(\mathbf{Y})$, we can firstly derive the solution of Lyapunov equation from (\ref{p2}) as
\begin{equation}
\begin{aligned}
\mathrm{vec}(\mathbf{W}^{(k)}_0) 
& = |\psi_{k}|^{2}(\mathbf{S}^\top \otimes \mathbf{S}^\top)\mathrm{vec}(\mathbf{W}^{(k)}_0) + \vec{\mathbf{1}}\\
&=(\mathbf{I}-|\psi_{k}|^{2}\mathbf{S}^\top \otimes \mathbf{S}^\top)^{-1}\vec{\mathbf{1}},
\end{aligned}
\label{reg_iir}
\end{equation}
where $\vec{\mathbf{1}}=[\mathbf{e}_1^\top,\dots,\mathbf{e}_N^\top]^\top$. Then, by using the trace property \(\mathrm{tr}(\mathbf{XY}) = \mathrm{tr}(\mathbf{YX})\), Lyapunov equation $|\psi_{k}|^{2}\mathbf{S}^\top\mathbf{W}^{(k)}_0 \mathbf{S} = \mathbf{W}^{(k)}_0-\mathbf{I}$, the noise gain from (\ref{p2}) can be rewritten as 
\begin{equation}
\begin{aligned}
G^{(k)} = \mathrm{tr}\big(\mathbf{I}\times(\mathbf{W}^{(k)}_0-\mathbf{I})\big), 
\end{aligned}
\label{h_iir}
\end{equation}
which can be further simplified as
\begin{equation}
\begin{aligned}
G^{(k)} & =\mathrm{vec}(\mathbf{I})^\top\mathrm{vec}(\mathbf{W}^{(k)}_0-\mathbf{I})\\
& =\vec{\mathbf{1}}^\top\big((\mathbf{I}-|\psi_{k}|^{2}\mathbf{S}^\top \otimes \mathbf{S}^\top)^{-1}-\mathbf{I}\big)\vec{\mathbf{1}}
\end{aligned}
\label{rreg_iir}
\end{equation}
by using $\mathrm{tr}(\mathbf{XY})=\mathrm{vec}(\mathbf{X}^\top)^\top\mathrm{vec}(\mathbf{Y})$ and substituting the solution of $\mathrm{vec}(\mathbf{W}^{(k)}_0) $ from (\ref{reg_iir}). By aggregating $G^{(k)}$ and omitting the constant, the regularizer $\mathcal{R}(\boldsymbol{\psi})$ is derived as
\begin{equation}
\begin{aligned}
&\mathcal{R}(\boldsymbol{\psi})=\vec{\mathbf{1}}^\top\left(\sum^{K}_{k=1}(\mathbf{I}-|\psi_{k}|^{2}\mathbf{S}^\top \otimes \mathbf{S}^\top)^{-1}\right)\vec{\mathbf{1}}.
\end{aligned}
\label{rrrerrg_iir}
\end{equation}
By the mixed-product property $(\mathbf{X}\otimes\mathbf{Y})(\mathbf{Z}\otimes\mathbf{W})=(\mathbf{X}\mathbf{Z}\otimes\mathbf{Y}\mathbf{W})$ and $\mathbf{S} =\mathbf{U}\mathbf{\Lambda}\mathbf{U}^{-1}$, it can be obtained $\mathbf{S}^\top \otimes \mathbf{S}^\top = \left( \mathbf{U}^{-1} \otimes \mathbf{U}^{-1} \right)^\top(\mathbf{\Lambda} \otimes \mathbf{\Lambda})\left( \mathbf{U} \otimes \mathbf{U} \right)^\top$. Then (\ref{rrrerrg_iir}) is derived as
\begin{equation}
\begin{aligned}
\mathcal{R}(\boldsymbol{\psi})
&=\vec{\mathbf{1}}^\top\left( \mathbf{U}^{-1} \otimes \mathbf{U}^{-1} \right)^\top\left( \sum_{k=1}^{K} (\mathbf{I} - |\psi_{k}|^{2} \mathbf{\Lambda} \otimes \mathbf{\Lambda})^{-1} \right)\left( \mathbf{U} \otimes \mathbf{U} \right)^\top\vec{\mathbf{1}}.
\end{aligned}
\label{rrreg_iir}
\end{equation}
Since $\mathbf{\Lambda}=\mathrm{diag}(\lambda_{1}, \dots, \lambda_{N})$, then we have   
\begin{equation}
\begin{aligned}
\left[\left( \sum_{k=1}^{K} (\mathbf{I} - |\psi_{k}|^{2} \mathbf{\Lambda} \otimes \mathbf{\Lambda})^{-1}\right) \right]_{i'i'}=\sum_{k=1}^{K}\frac{1}{1 - |\psi_{k}|^{2}\lambda_{i}\lambda_{j}},
\end{aligned}
\label{rrrssseg_iirr}
\end{equation}
where $i'=(i-1)N+j$ and the other entries are all zero. By using the diagonal‐vector product identity $\mathbf{a}^\top\mathrm{diag}(\mathbf{b})\mathbf{c}=\mathbf{b}(\mathbf{a}\odot\mathbf{c})$ with  
\begin{equation}
\begin{aligned}
\mathbf{a}^\top  = \vec{\mathbf{1}}^\top\left( \mathbf{U}^{-1} \otimes \mathbf{U}^{-1} \right)^\top, \mathbf{c} = \left( \mathbf{U} \otimes \mathbf{U} \right)^\top\vec{\mathbf{1}}\\
\mathrm{diag}(\mathbf{b})  = \left( \sum_{k=1}^{K} (\mathbf{I} - |\psi_{k}|^{2} \mathbf{\Lambda} \otimes \mathbf{\Lambda})^{-1}\right),
\end{aligned}
\label{rrreg_iirr}
\end{equation}
(\ref{r3_final}) follows. The proof is completed.

\subsection{Proof of Proposition 3}
\label{app:p3}
 By substitution and total expectation property \(\mathbb{E}[\mathbf{X}] = \mathbb{E}[\mathbb{E}[\mathbf{X}|\mathbf{Y}]]\), it shows
    \begin{equation}
\begin{aligned}
\mathbb{E}[\widetilde{\mathbf{y}}_{\mathbf{n}_{t-1}}\widetilde{\mathbf{y}}^\mathrm{H}_{\mathbf{n}_{t-1}}]=\mathbb{E}\Big[\mathbf{H}_{\mathbf{n}_{t-1}} (\mathcal{\Phi}_{t:T-1})(\mathbf{S}_{t-1}-\mathbf{D}_{t-1})\mathbb{E}[\mathbf{n}_{t-1}\mathbf{n}^\mathrm{H}_{t-1}](\mathbf{S}_{t-1}-\mathbf{D}_{t-1})^\mathrm{H}\mathbf{H}^\mathrm{H}_{\mathbf{n}_{t-1}} (\mathcal{\Phi}_{t:T-1})\Big].
  \end{aligned}
\label{p311}
\end{equation}
As \(\mathbb{E}[\mathbf{n}_{t-1}\mathbf{n}^\mathrm{H}_{t-1}]=\sigma_{t-1}^2\mathbf{I}\), we can further obtain 
    \begin{equation}
\begin{aligned}
\mathbb{E}[\widetilde{\mathbf{y}}_{\mathbf{n}_{t-1}}\widetilde{\mathbf{y}}^\mathrm{H}_{\mathbf{n}_{t-1}}]=\sigma_{t-1}^2\mathbb{E}\Big[\mathbf{H}_{\mathbf{n}_{t-1}} (\mathcal{\Phi}_{t:T-1})(\mathbf{S}_{t-1}-\mathbf{D}_{t-1})(\mathbf{S}_{t-1}-\mathbf{D}_{t-1})^\mathrm{H}\mathbf{H}^\mathrm{H}_{\mathbf{n}_{t-1}} (\mathcal{\Phi}_{t:T-1})\Big].
  \end{aligned}
\label{p312}
\end{equation}
By using real-valued \(\mathbf{X^\mathrm{H}}=\mathbf{X^\top}\), \(\mathrm{tr}(\mathbb{E}[\mathbf{X}])=\mathbb{E}[\mathrm{tr}(\mathbf{X})]\), \(\mathrm{tr}(\mathbf{XY})=\mathrm{tr}(\mathbf{YX})\), we cyclically reorder the terms to obtain 
    \begin{equation}
\begin{aligned}
\mathrm{tr}\bigg(\mathbb{E}[\widetilde{\mathbf{y}}_{\mathbf{n}_{t-1}}\widetilde{\mathbf{y}}^\mathrm{H}_{\mathbf{n}_{t-1}}]\bigg)=\sigma_{t-1}^2\mathbb{E}\bigg[\mathrm{tr}\Big((\mathbf{S}_{t-1}-\mathbf{D}_{t-1})^\mathrm{H}\mathbf{H}^\mathrm{H}_{\mathbf{n}_{t-1}} (\mathcal{\Phi}_{t:T-1})\mathbf{H}_{\mathbf{n}_{t-1}} (\mathcal{\Phi}_{t:T-1})(\mathbf{S}_{t-1}-\mathbf{D}_{t-1})\Big)\bigg].
  \end{aligned}
\label{p313}
\end{equation}
As $\mathbf{G}_{\mathbf{n}_{t-1}}(\mathcal{\Phi}_{t:T-1})=\mathbf{H}^\top_{\mathbf{n}_{t-1}} (\mathcal{\Phi}_{t:T-1})\mathbf{H}_{\mathbf{n}_{t-1}} (\mathcal{\Phi}_{t:T-1})$, \(\mathbb{E}[\mathrm{tr}(\mathbf{X})]=\mathrm{tr}(\mathbb{E}[\mathbf{X}]) \) and \(\mathbb{E}[\mathbf{S}_{t-1}]=\overline{\mathbf{S}}\), (\ref{p3}) follows. The proof is completed.

\subsection{Proof of Remark 5}
\label{app:r5}
By applying $\mathrm{tr}(\mathbf{XY})\le\Vert\mathbf{X} \Vert_{2}\mathrm{tr}(\mathbf{Y})$,  \(\mathrm{tr}(\mathbb{E}[\mathbf{X}]) = \mathbb{E}[\mathrm{tr}(\mathbf{X})]\), and $\mathrm{tr}(\mathbf{XX^\top})\le \mathrm{rank}(\mathbf{X})\Vert\mathbf{X}\Vert^{2}_{2}$, $G_{t-1}$ from~(\ref{p3}) can be bounded, satisfies the following chain: 
\begin{equation}
\begin{aligned}
G_{t-1} 
& = \mathrm{tr}\Big(\mathbb{E}\big[\mathbf{G}_{\mathbf{n}_{t-1}}(\mathcal{\Phi}_{t:T-1}) \mathbf{S}_{t-1}\mathbf{S}^\top_{t-1}\big]\Big) \\
& \le \Big\Vert\mathbb{E}\big[\mathbf{G}_{\mathbf{n}_{t-1}}(\mathcal{\Phi}_{t:T-1})\big] \Big\Vert_{2}\mathrm{tr}\Big(\mathbb{E}\big[\mathbf{S}_{t-1}\mathbf{S}^\top_{t-1}\big]\Big)\\
& \le N\rho^{2}_{\star}\Big\Vert\mathbb{E}\big[\mathbf{G}_{\mathbf{n}_{t-1}}(\mathcal{\Phi}_{t:T-1})\big] \Big\Vert_{2},
\end{aligned}
\label{r5}
\end{equation}
with $\mathrm{rank}(\mathbf{S}_{t-1})\le N$ and $\Vert\mathbf{S}_{t-1}\Vert_{2}\le \rho_{\star}$. Then 
 based on the spectral norm subadditivity $\Vert\mathbf{X}+\mathbf{Y} \Vert_{2} \le \Vert\mathbf{X} \Vert_{2}+\Vert\mathbf{Y} \Vert_{2}$, Jensen’s inequality $\big\Vert\mathbb{E}[\mathbf{X}]\big\Vert_{2}\le \mathbb{E}\big[\Vert\mathbf{X}\Vert_{2}\big]$, and the expression in (\ref{ker}), (\ref{r5}) can be further bounded as 
\begin{equation}
\begin{aligned}
G_{t-1}& \le N\rho^{2}_{\star}\left(\sum_{\tau_1 =t}^{T}\sum_{\tau_2 =t}^{T}\Big\Vert\phi_{\tau_1}\phi_{\tau_2}\mathbb{E}\big[\mathcal{\Phi}_{t:\tau_1-1}^\top\mathcal{\Phi}_{t:\tau_2-1}\big]\Big\Vert_{2}\right)\\
& \le N\rho^{2}_{\star}\left(\sum_{\tau_1 =t}^{T}\sum_{\tau_2 =t}^{T}|\phi_{\tau_1}\phi_{\tau_2}|\mathbb{E}\Big[\big\Vert\mathcal{\Phi}_{t:\tau_1-1}^\top\mathcal{\Phi}_{t:\tau_2-1}\big\Vert_{2}\Big]\right).
\end{aligned}
\label{r5_}
\end{equation}
From (\ref{ker})-(\ref{vec}) and the sub‐multiplicativity of the spectral norm \(\|\mathbf{XY}\|_{2} \leq \|\mathbf{X}\|_{2} \|\mathbf{Y}\|_{2}\), we find a new bound for (\ref{r5_}) 
\begin{equation}
\begin{aligned}
G_{t-1} & \le N\rho^{2}_{\star}\left(\sum_{\tau_1 =t}^{T}\sum_{\tau_2 =t}^{T}|\phi_{\tau_1}\phi_{\tau_2}|\Vert\mathbf{S}_{t}\Vert^{2l}_{2} \Vert\overline{\mathbf{S}}^{|\tau_2-\tau_1|}\Vert_{2}\right)\\
& \le N\rho^{2}_{\star}\left(\sum_{\tau_1 =t}^{T}\sum_{\tau_2 =t}^{T}|\phi_{\tau_1}\phi_{\tau_2}|p^{|\tau_2-\tau_1|}\rho^{\tau_2+\tau_1-2t}_{\star}\right)
\end{aligned}
\label{r555}
\end{equation}
by $l= \textnormal{min} \{\tau_1,\tau_2\}-t$ and $\Vert\overline{\mathbf{S}}\Vert_{2}=p\rho_{\star}$. Then vector form of (\ref{r555}) can be easily written as
\begin{equation}
\begin{aligned}
G_{t-1} \le N\rho^{2}_{\star}\left(\boldsymbol{\phi'}^\top \mathbf{V'}_{t-1}\boldsymbol{\phi'}\right),
\end{aligned}
\label{r55555}
\end{equation}
owing to the fact that $[\boldsymbol{\phi'}]_{\tau_1+1}=\vert\phi_{\tau_1}\vert,[\boldsymbol{\phi'}]_{\tau_2+1}=\vert\phi_{\tau_2}\vert$, and $\big[\mathbf{V'}_{t-1}\big]_{(\tau_1+1)(\tau_2+1)}=p^{|\tau_2-\tau_1|}\rho^{\tau_2+\tau_1-2t}_{\star}$. As $\tau_1,\tau_2>t$, all remaining entries in $\mathbf{V'}_{t-1}$ are zero. Then after the summation over $t$, (\ref{r5_final}) follows. The proof is completed.

\subsection{Proof of Proposition 4}
\label{app:p4}
With $\widetilde{\mathbf{w}}^{(k)}_{0}=\mathbf{0}$, analogous to (\ref{p211}), $\widetilde{\mathbf{w}}^{(k)}_{t}$ satisfies
\begin{equation}
\begin{aligned}
\widetilde{\mathbf{w}}^{(k)}_{t} = \sum_{\tau=0}^{t-1}\psi_{k}^{t-\tau-1}\mathcal{\Phi}_{\tau+1:t-1}(\psi_{k}\mathbf{S}_{\tau}-\mathbf{D}^{(k)})\mathbf{n}^{(k)}_{\tau}.
 \end{aligned}
\label{pp4}
\end{equation}
Then  $\mathbb{E} [\widetilde{\mathbf{w}}^{(k)}_{t} (\widetilde{\mathbf{w}}^{(k)}_{t})^\mathrm{H} ]$ is given by 
\begin{equation}
\begin{aligned}
&\mathbb{E} [\widetilde{\mathbf{w}}^{(k)}_{t} (\widetilde{\mathbf{w}}^{(k)}_{t})^\mathrm{H} ] = \mathbb{E}\bigg[\sum_{\tau_1=0}^{t-1}\sum_{\tau_2=0}^{t-1}\psi_{k}^{t-\tau_1-1}\mathcal{\Phi}_{\tau_1+1:t-1}(\psi_{k}\mathbf{S}_{\tau_1}-\mathbf{D}^{(k)})\mathbf{n}^{(k)}_{\tau_1}(\mathbf{n}^{(k)}_{\tau_2})^\mathrm{H} (\psi_{k}\mathbf{S}_{\tau_2}-\mathbf{D}^{(k)})^\mathrm{H}\mathcal{\Phi}^\mathrm{H}_{\tau_2+1:t-1}(\psi^{*}_{k})^{t-\tau_2-1}\bigg].
 \end{aligned}
\label{pp4ss}
\end{equation}
Since \(\mathbb{E}[\mathbf{X}] = \mathbb{E}[\mathbb{E}[\mathbf{X}|\mathbf{Y}]]\), $\mathbb{E}[\mathbf{n}^{(k)}_{\tau_1}(\mathbf{n}^{(k)}_{\tau_2})^\mathrm{H}] = (\sigma^{(k)})^2 \mathbf{I}$ if $\tau_1=\tau_2$ and $\mathbf{0}$ if $\tau_1\neq\tau_2$, and real-valued \(\mathbf{X^\mathrm{H}}=\mathbf{X^\top}\), (\ref{pp4ss}) can be rewritten as 
\begin{equation}
\begin{aligned}
\mathbb{E} [\widetilde{\mathbf{w}}^{(k)}_{t} (\widetilde{\mathbf{w}}^{(k)}_{t})^\mathrm{H} ] &= (\sigma^{(k)})^2\sum_{\tau=0}^{t-1}|\psi_{k}|^{2(t-\tau-1)}\mathbb{E}\Big[\mathcal{\Phi}_{\tau+1:t-1} (\psi_{k}\mathbf{S}_{\tau}-\mathbf{D}^{(k)})(\psi_{k}\mathbf{S}_{\tau}-\mathbf{D}^{(k)})^\mathrm{H}\mathcal{\Phi}^\top_{\tau+1:t-1}\Big].
 \end{aligned}
\label{pp4sss}
\end{equation}
By \(\mathrm{tr}(\mathbb{E}[\mathbf{X}]) = \mathbb{E}[\mathrm{tr}(\mathbf{X})]\), \(\mathrm{tr}(\mathbf{XY}) = \mathrm{tr}(\mathbf{YX})\), \(\mathbb{E}[\mathrm{tr}(\mathbf{X})]=\mathrm{tr}(\mathbb{E}[\mathbf{X}]) \), we can write 
\begin{equation}
\begin{aligned}
\mathrm{tr}&\bigg(\mathbb{E} [\widetilde{\mathbf{w}}^{(k)}_{t} (\widetilde{\mathbf{w}}^{(k)}_{t})^\mathrm{H} ]\bigg) = (\sigma^{(k)})^2\mathrm{tr}\bigg(\sum_{\tau=0}^{t-1}|\psi_{k}|^{2(t-\tau-1)}\mathbb{E}\Big[(\psi_{k}\mathbf{S}_{\tau}-\mathbf{D}^{(k)})^\mathrm{H}\mathcal{\Phi}^\top_{\tau+1:t-1}\mathcal{\Phi}_{\tau+1:t-1}(\psi_{k}\mathbf{S}_{\tau}-\mathbf{D}^{(k)})\Big]\bigg).
 \end{aligned}
\label{pp4sss44}
\end{equation}
Since \(\mathcal{\Phi}_{\tau+1:t-1}\) and \(\mathbf{S}_{\tau}\) are independent in the summation for every $\tau$, we can rewrite (\ref{pp4sss44}) using total expectation \(\mathbb{E}[\mathbf{X}] = \mathbb{E}[\mathbb{E}[\mathbf{X}|\mathbf{Y}]]\) and linearity $\mathbb{E}[\mathbf{X}+\mathbf{Y}] = \mathbb{E}[\mathbf{X}]+\mathbb{E}[\mathbf{Y}]$ as
\begin{equation}
\begin{aligned}
&\mathrm{tr}\bigg(\mathbb{E} [\widetilde{\mathbf{w}}^{(k)}_{t} (\widetilde{\mathbf{w}}^{(k)}_{t})^\mathrm{H} ]\bigg) = (\sigma^{(k)})^2\mathrm{tr}\bigg(\mathbb{E}\Big[(\psi_{k}\mathbf{S}_{\star}-\mathbf{D}^{(k)})^\mathrm{H}\mathbb{E}\Big[\sum_{\tau=0}^{t-1}|\psi_{k}|^{2(t-\tau-1)}\mathcal{\Phi}^\top_{\tau+1:t-1}\mathcal{\Phi}_{\tau+1:t-1}\Big](\psi_{k}\mathbf{S}_{\star}-\mathbf{D}^{(k)})\Big]\bigg),
 \end{aligned}
\label{pp4sww44}
\end{equation}
where \(\mathbf{S}_\star\) denotes a random graph shift operator that is not associated with any specific time stamp but is defined solely by its independence properties. Then, let $t\to \infty$, it holds
\begin{equation}
\begin{aligned}
\lim_{t\to \infty}\mathbb{E}\Big[\sum_{\tau=0}^{t-1}|\psi_{k}|^{2(t-\tau-1)}\mathcal{\Phi}^\top_{\tau+1:t-1}\mathcal{\Phi}_{\tau+1:t-1}\Big]\to \mathbf{W}^{(k)}_{\mathcal{\Phi}}=\sum_{\tau=0}^{\infty}\mathbb{E}\big[|\psi_{k}|^{2\tau}\mathcal{\Phi}_{0:\tau-1}^\top\mathcal{\Phi}_{0:\tau-1}\big].
 \end{aligned}
\label{pp}
\end{equation}
Finally, let $\mathbf{S}_{\star}$ be instantiated as $\mathbf{S}_{t}$ and follow the expansion rule (\ref{p215}). Then (\ref{p4}) follows. The proof is completed.

\subsection{Proof of Lemma 1}
\label{app:l1}
We shall prove the convergence of \(\mathbf{W}^{(k)}_{\mathcal{\Phi}}\) from the spectral norm analysis. By applying Jensen's inequality $\big\Vert\mathbb{E}[\mathbf{X}]\big\Vert_{2}\le \mathbb{E}\big[\Vert\mathbf{X}\Vert_{2}\big]$, any subterm of \(\mathbf{W}^{(k)}_{\mathcal{\Phi}}\) satisfies
\begin{equation}
\begin{aligned}
\Big\Vert \mathbb{E} \big[|\psi_{k}|^{2\tau} \mathcal{\Phi}_{0:\tau-1}^\top \mathcal{\Phi}_{0:\tau-1}\big] \Big\Vert_{2} \leq \mathbb{E} \Big[\big\Vert |\psi_{k}|^{2\tau} \mathcal{\Phi}_{0:\tau-1}^\top\mathcal{\Phi}_{0:\tau-1} \big\Vert_{2}\Big].
\end{aligned}
\end{equation}
From the sub‐multiplicativity of the spectral norm \(\|\mathbf{XY}\|_{2} \leq \|\mathbf{X}\|_{2} \|\mathbf{Y}\|_{2}\), we can obtain  
\begin{equation}
\begin{aligned}
\mathbb{E} \Big[\big\Vert |\psi_{k}|^{2\tau} \mathcal{\Phi}_{0:\tau-1}^\top\mathcal{\Phi}_{0:\tau-1} \big\Vert_{2}\Big] \leq \mathbb{E} \Big[(\|\psi_{k} \mathbf{S}_{t}\|_{2})^{2\tau}\Big].
\end{aligned}
\end{equation}
Since \(\|\psi_k \mathbf{S}_{t}\|_{2} \leq |\psi_k \rho_{\star}| < 1\), it follows that
\begin{equation}
\begin{aligned}
\Big\Vert \mathbb{E} \big[|\psi_{k}|^{2\tau} \mathcal{\Phi}_{0:\tau-1}^\top \mathcal{\Phi}_{0:\tau-1}\big] \Big\Vert_{2} \leq \mathbb{E} \Big[(\|\psi_{k} \mathbf{S}_{t}\|_{2})^{2\tau}\Big] \leq |\psi_k \rho_{\star}|^{2\tau}.
\end{aligned}
\end{equation}
Since \(0 < |\psi_k \rho_{\star}| < 1\), the term \(|\psi_k \rho_{\star}|^{2\tau}\) converges. Then, $\mathbf{W}^{(k)}_{\mathcal{\Phi}}  = \sum_{\tau=0}^{\infty} \mathbb{E} \left[|\psi_{k}|^{2\tau} \mathcal{\Phi}_{0:\tau-1}^\top  \mathcal{\Phi}_{0:\tau-1}\right]$ converges. Finally, by applying total expectation \(\mathbb{E}[\mathbf{X}] = \mathbb{E}[\mathbb{E}[\mathbf{X}|\mathbf{Y}]]\), (\ref{l1}) follows. The proof is completed.

\subsection{Proof of Lemma 2}
\label{app:l2}
Leveraging total expectation \(\mathbb{E}[\mathbf{X}] = \mathbb{E}[\mathbb{E}[\mathbf{X}|\mathbf{Y}]]\), the term 
\begin{equation}
\begin{aligned}
& \mathbb{E}\bigg[\prod_{t'=1}^{\tau\leftarrow}\big(\mathbf{I}+(\psi^{*}_k\mathbf{S}^\top - \mathbf{I})\mathbf{P}_{\mathcal{T}_{t'}}\big)\times \prod_{t'=1}^{\rightarrow\tau}\big(\mathbf{I}+\mathbf{P}_{\mathcal{T}_{t'}} (\psi_k\mathbf{S} - \mathbf{I})\big)\bigg]\\
& = \mathbb{E}\bigg[ \prod_{t'=2}^{\tau\leftarrow}\big(\mathbf{I}+(\psi^{*}_k\mathbf{S}^\top - \mathbf{I})\mathbf{P}_{\mathcal{T}_{t'}}\big)\times \mathbb{E}\Big[\big(\mathbf{I}+(\psi^{*}_k\mathbf{S}^\top - \mathbf{I})\mathbf{P}_{\mathcal{T}_{1}}\big)\\ 
&\times\big(\mathbf{I}+\mathbf{P}_{\mathcal{T}_{1}} (\psi_k\mathbf{S} - \mathbf{I})\big)\Big]
\prod_{t'=2}^{\rightarrow\tau}\big(\mathbf{I}+\mathbf{P}_{\mathcal{T}_{t'}} (\psi_k\mathbf{S} - \mathbf{I})\big)\bigg].
\end{aligned}
\label{TERM}
\end{equation}
Then the middle expectation can be written as 
\begin{equation}
\begin{aligned}
 &\mathbb{E}\Big[\big(\mathbf{I}+(\psi^{*}_k\mathbf{S}^\top - \mathbf{I})\mathbf{P}_{\mathcal{T}_{1}}\big)\big(\mathbf{I}+\mathbf{P}_{\mathcal{T}_{1}} (\psi_k\mathbf{S} - \mathbf{I})\big)\Big]\\
 & = \mathbb{E}\Big[\mathbf{I}+|\psi_{k}|^{2}\mathbf{S}^\top\mathbf{P}_{\mathcal{T}_{1}}\mathbf{S}-\mathbf{P}_{\mathcal{T}_{1}}\Big]\\
 & = \mathbf{I}+|\psi_{k}|^{2}\mathbf{S}^\top\mathbf{P}\mathbf{S}-\mathbf{P},
\end{aligned}
\end{equation}
with the fact $\mathbf{P}_{\mathcal{T}_{t}}\mathbf{P}_{\mathcal{T}_{t}} = \mathbf{P}_{\mathcal{T}_{t}}$ and $\mathbb{E}[\mathbf{P}_{\mathcal{T}_t}] = \mathbf{P}$. Since $\mathbf{P} = p\mathbf{I}$ and \(\Vert \psi_{k}\mathbf{S}\Vert_{2}<1 \to |\psi_{k}|^{2}\mathbf{S}^\top\mathbf{S}\prec \mathbf{I}\), we can have 
\begin{equation}
\begin{aligned}
 &\mathbf{I}+|\psi_{k}|^{2}\mathbf{S}^\top\mathbf{P}\mathbf{S}-\mathbf{P}=\mathbf{I}+p|\psi_{k}|^{2}\mathbf{S}^\top\mathbf{S}-p\mathbf{I}\prec\mathbf{I}.
\end{aligned}
\label{loose}
\end{equation}
This means $\exists \epsilon \in (0,1)$ makes (\ref{loose}) tight
\begin{equation}
\begin{aligned}
 \mathbf{I}+p|\psi_{k}|^{2}\mathbf{S}^\top\mathbf{S}-p\mathbf{I}\preceq\epsilon\mathbf{I}.
\end{aligned}
\label{tight}
\end{equation}
By substituting (\ref{tight}) into (\ref{TERM}) and iteratively applying total expectation \(\mathbb{E}[\mathbf{X}] = \mathbb{E}[\mathbb{E}[\mathbf{X}|\mathbf{Y}]]\), we have 
\begin{equation}
\begin{aligned}
& \mathbb{E}\bigg[\prod_{t'=1}^{\tau\leftarrow}\big(\mathbf{I}+(\psi^{*}_k\mathbf{S}^\top - \mathbf{I})\mathbf{P}_{\mathcal{T}_{t'}}\big) \times \prod_{t'=1}^{\rightarrow\tau}\big(\mathbf{I}+\mathbf{P}_{\mathcal{T}_{t'}} (\psi_k\mathbf{S} - \mathbf{I})\big)\bigg] \preceq\epsilon^{\tau}\mathbf{I}.
\end{aligned}
\label{upterm}
\end{equation}
As $\sum^{\infty}_{\tau=0}\epsilon^{\tau}\mathbf{I}$ is a geometric series that absolutely converges, then $\mathbf{W}^{(k)}_{\mathbf{P}}$ converges. Then by total expectation \(\mathbb{E}[\mathbf{X}] = \mathbb{E}[\mathbb{E}[\mathbf{X}|\mathbf{Y}]]\) again, (\ref{l2}) follows. The proof is completed. 
\subsection{Proof of Remark 7}
\label{app:r7}
With the symmetry, $\Vert\mathbf{S}\Vert_{2}=\rho_{\star}$ holds. Then we can specify $\epsilon$ in (\ref{tight}) as
\begin{equation}
\begin{aligned}
\epsilon=\rho(\mathbf{I}+p|\psi_{k}|^{2}\mathbf{S}^\top\mathbf{S}-p\mathbf{I})= 1+p|\psi_{k}|^{2}\rho^{2}_{\star}-p.
\end{aligned}
\label{epss}
\end{equation}
By substituting (\ref{epss}) into (\ref{upterm}), we have
\begin{equation}
\begin{aligned}
\mathbf{W}^{(k)}_{\mathbf{P}} \preceq \sum^{\infty}_{\tau=0}(1+p|\psi_{k}|^{2}\rho^{2}_{\star}-p)^{\tau}\mathbf{I}=\frac{1}{p-p|\psi_{k}|^{2}\rho^{2}_{\star}}\mathbf{I}.
\end{aligned}
\label{upespp}
\end{equation}
Then, it shows 
\begin{equation}
\begin{aligned}
\mathbb{E}[\mathbf{P}_{\mathcal{T}_t}\mathbf{W}^{(k)}_{\mathbf{P}}\mathbf{P}_{\mathcal{T}_t}] \preceq \frac{1}{p-p|\psi_{k}|^{2}\rho^{2}_{\star}}\mathbb{E}[\mathbf{P}_{\mathcal{T}_t}\mathbf{P}_{\mathcal{T}_t}]=\frac{1}{1-|\psi_{k}|^{2}\rho^{2}_{\star}}\mathbf{I},
\end{aligned}
\label{upespp_2}
\end{equation}
with $\mathbf{P}_{\mathcal{T}_{t}}\mathbf{P}_{\mathcal{T}_{t}} = \mathbf{P}_{\mathcal{T}_{t}}$ and $\mathbb{E}[\mathbf{P}_{\mathcal{T}_t}] = \mathbf{P} =p\mathbf{I}$. Then by taking the spectral norm of (\ref{upespp_2}), (\ref{r7_final}) follows. The proof is completed. 

\section{Detailed Derivation of QEF-ATC for Regression over Networks}
In this section, we provide the detailed derivations of QEF-ATC for Section \ref{s5b}.
\subsection{Corollary of Theorem 2}
\label{Appb1}
Based on the noise propagation model by (\ref{Noiseppp}), its error feedback version
\begin{equation}
\begin{aligned}
\boldsymbol{\widetilde{\xi}}_{t} &= \mathcal{B}\mathcal{A}\boldsymbol{\widetilde{\xi}}_{t-1} + \mathcal{B}(\mathcal{A}-\mathcal{D})\boldsymbol{n}_{t-1},
\end{aligned}
\label{qefr1}
\end{equation}
where $\mathcal{D}=\mathrm{blkdiag}(\mathbf{D}_{1},\dots,\mathbf{D}_{N})$ and $\mathbf{D}_i\in\mathbb{R}^{M\times M}$. Analogous to Proposition \ref{prop2}, the total noise gain
\begin{equation}
    \begin{aligned}
     G_{\xi}+I_{\xi} & = \mathrm{tr} \Big( (\mathcal{A}-\mathcal{D})^\mathrm{H}\mathcal{B}^\mathrm{H}\mathbf{W}_{\xi}\mathcal{B}(\mathcal{A}-\mathcal{D})\Big)
    \end{aligned}
    \label{gig}
\end{equation}
with $G_{\xi}$ is the original noise gain and $I_{\xi}$ is the mitigation, respectively. By parameterizing $\mathbf{D}_i=\alpha_{i}\mathbf{I}$, $\mathcal{D}$ becomes a diagonal matrix and is linked to Theorem \ref{Th2}. Therefore, we can directly have the Corollary:
\begin{corollary}
    Given the propagation of noise (\ref{qefr1}) and the gain (\ref{gig}), the feedback coefficient $\alpha_i$ from $\mathbf{D}_i=\alpha_{i}\mathbf{I}$ satisfies 
    \begin{equation}
    \begin{aligned}
     \alpha_i & = \frac{\sum_{j=1}^{M}[\mathcal{B}^\mathrm{H}\mathbf{W}_{\xi}\mathcal{B}\mathcal{A}]_{\big((i-1)M +j\big)\big((i-1)M +j\big)}}{\sum_{j=1}^{M}[\mathcal{B}^\mathrm{H}\mathbf{W}_{\xi}\mathcal{B}]_{\big((i-1)M +j\big)\big((i-1)M +j\big)}}.
    \end{aligned}
    \label{gigss}
\end{equation}
\end{corollary}
\subsection{Equivalent Form of Algorithm}
\label{Appb2}
Here we show the QEF-ATC with differential quantization for the regression task. Based on (\ref{qefr1})-(\ref{gigss}), the explicit form of the quantized version of (\ref{ATC}) with error feedback is given by:
\begin{equation}
    \begin{aligned}
    \xi_{i,t} &= \mathbf{x}_{i,t-1} - \mu\nabla f_{i}(\mathbf{x}_{i,t-1})\\
    \hat{\xi}_{i,t}&=\hat{\xi}_{i,t-1}+\mathcal{Q}[\xi_{i,t}-\hat{\xi}_{i,t-1}] \\
    \mathbf{x}_{i,t} &= \sum_{j\in\mathcal{N}_{i}}w_{ij}\hat{\xi}_{i,t} - \alpha_i\mathbf{n}_{i,t-1}  
    \end{aligned}
\end{equation}
where the quantization noise is defined as $\mathbf{n}_{i,t-1} = (\xi_{i,t}-\hat{\xi}_{i,t-1})-\mathcal{Q}[\xi_{i,t}-\hat{\xi}_{i,t-1}]$. Notably, we apply error feedback during the combination step, as opposed to the more common approach of gradient compensation. This demonstrates that the benefits of error feedback are not limited to gradient compression.  
\end{document}